%% file: main.tex
\pdfoutput=1
\PassOptionsToPackage{noend}{algorithm2e}
\documentclass[pmlr,twoside,tablecaptionbottom]{pgm-jmlr}
\usepackage{pgm-pmlr-dat}
\firstpageno{1}
\SetKw{Continue}{continue}

\usepackage[utf8]{inputenc}
\inputencoding{utf8}

\usepackage{booktabs}
\usepackage{acronym}
\usepackage[capitalise]{cleveref}
\usepackage{csquotes}
\usepackage{tikz}
\usetikzlibrary{arrows,arrows.meta,automata,backgrounds,calc,patterns,positioning,shapes,shadows}

\input{defs.tex}

\title{Efficient Detection of Commutative Factors in Factor Graphs}

\author{
	\Name{Malte Luttermann} \Email{malte.luttermann@dfki.de} \\
	\addr German Research Center for Artificial Intelligence (DFKI), Lübeck, Germany
	\AND
	\Name{Johann Machemer} \Email{johann.machemer@student.uni-luebeck.de} \\
	\addr Institute for Software Engineering and Programming Languages, University of Lübeck, Germany
	\AND
	\Name{Marcel Gehrke} \Email{marcel.gehrke@uni-hamburg.de} \\
	\addr Institute for Humanities-Centered Artificial Intelligence, University of Hamburg, Germany
}

\begin{document}

\maketitle

\begin{abstract}
	Lifted probabilistic inference exploits symmetries in probabilistic graphical models to allow for tractable probabilistic inference with respect to domain sizes.
	To exploit symmetries in, e.g., \aclp{fg}, it is crucial to identify commutative factors, i.e., factors having symmetries within themselves due to their arguments being exchangeable.
	The current state of the art to check whether a factor is commutative with respect to a subset of its arguments iterates over all possible subsets of the factor's arguments, i.e., $O(2^n)$ iterations for a factor with $n$ arguments in the worst case.
	In this paper, we efficiently solve the problem of detecting commutative factors in a \acl{fg}.
	In particular, we introduce the \emph{\acf{decor}} algorithm, which allows us to drastically reduce the computational effort for checking whether a factor is commutative in practice.
	We prove that \ac{decor} efficiently identifies restrictions to drastically reduce the number of required iterations and validate the efficiency of \ac{decor} in our empirical evaluation.
\end{abstract}

\acresetall

\section{Introduction} \label{sec:com_intro}
Probabilistic graphical models provide a well-founded formalism to reason under uncertainty and compactly encode a full joint probability distribution as a product of factors.
A fundamental task using probabilistic graphical models is to perform probabilistic inference, that is, to compute marginal distributions of \acp{rv} given observations for other \acp{rv}.
In general, however, probabilistic inference scales exponentially with the number of \acp{rv} in a propositional probabilistic model such as a \acl{bn}, a \acl{mn}, or a \ac{fg} in the worst case.
To allow for tractable probabilistic inference (e.g., probabilistic inference requiring polynomial time) with respect to domain sizes of \aclp{lv}, lifted inference algorithms exploit symmetries in a probabilistic graphical model by using a representative of indistinguishable individuals for computations~\citep{Niepert2014a}.
Performing lifted inference, however, requires a lifted representation of the probabilistic graphical model, which encodes equivalent semantics to the original propositional probabilistic model.
In particular, it is necessary to detect symmetries in the propositional model to construct the lifted representation.
Commutative factors, i.e., factors that map to the same potential value regardless of the order of a subset of their arguments, play a crucial role when detecting symmetries in an \ac{fg} and constructing a lifted representation of the \ac{fg}.
To detect commutative factors, the current state-of-the-art algorithm iterates over all possible subsets of a factor.
In this paper, we show that the search can be guided significantly more efficient and thereby we solve the problem of efficiently detecting commutative factors in an \ac{fg}.

\paragraph{Previous work.}
In probabilistic inference, lifting exploits symmetries in a probabilistic model, allowing to carry out query answering more efficiently while maintaining exact answers~\citep{Niepert2014a}.
First introduced by \citet{Poole2003a}, \acp{pfg} provide a lifted representation by combining first-order logic with probabilistic models, which can be utilised by \acl{lve} as a lifted inference algorithm operating on \acp{pfg}.
After the introduction of \ac{lve} by \citet{Poole2003a}, \ac{lve} has steadily been refined by many researchers~\citep{DeSalvoBraz2005a,DeSalvoBraz2006a,Milch2008a,Kisynski2009a,Taghipour2013a,Braun2018a}.
To perform lifted probabilistic inference, the lifted representation (e.g., the \ac{pfg}) has to be constructed first.
Currently, the state-of-the-art algorithm to construct a \ac{pfg} from a given \ac{fg} is the \ac{cpr} algorithm~\citep{Luttermann2024a}, which is a refinement of the \ac{cp} algorithm~\citep{Kersting2009a,Ahmadi2013a} and is able to construct a \ac{pfg} entailing equivalent semantics as the initially given \ac{fg}.
The \ac{cp} algorithm builds on work by \citet{Singla2008a} and incorporates a colour passing procedure to detect symmetries in a graph similar to the Weisfeiler-Leman algorithm~\citep{Weisfeiler1968a}, which is commonly used to test for graph isomorphism.
\Ac{cpr} employs a slightly refined colour passing procedure and during the course of the algorithm, \ac{cpr} checks for every factor in the given \ac{fg} whether the factor is commutative with respect to a subset of its arguments to determine the outgoing messages of the factor for the colour passing procedure.
The check for commutativity is implemented by iterating over all possible subsets of the factor's arguments, which results in $O(2^n)$ iterations for a factor with $n$ arguments in the worst case.

\paragraph{Our contributions.}
In this paper, we efficiently solve the problem of detecting commutative factors in an \ac{fg}, thereby speeding up the construction of an equivalent lifted representation for a given \ac{fg}.
More specifically, we show how so-called \emph{buckets} can be used to restrict the possible candidates of subsets that have to be considered when checking whether a factor is commutative---that is, which of its arguments are exchangeable.
The idea is that within these buckets, only a subset of potentials from the original factor has to be considered.
Further, argument(s), which produce certain potential(s) can be identified.
Thereby, the search space can be drastically reduced.
We prove that using buckets, the run time complexity is drastically reduced for many practical settings, making the check for commutativity feasible in practice.
Afterwards, we gather the theoretical insights and propose the \ac{decor} algorithm, which applies the theoretical insights to efficiently detect commutative factors in an \ac{fg}.
In addition to our theoretical results, we validate the efficiency of \ac{decor} in our empirical evaluation.

\paragraph{Structure of this paper.}
The remainder of this paper is structured as follows.
We first introduce necessary background information and notations.
In particular, we define \acp{fg} as well as the notion of commutative factors and provide background information on the \ac{cpr} algorithm.
Afterwards, we formally define the concept of a bucket and present our theoretical results, showing how buckets can be used to efficiently restrict possible candidates of subsets that have to be considered when searching for exchangeable arguments of a factor.
We then introduce the \ac{decor} algorithm, which applies the theoretical insights to efficiently detect commutative factors in an \ac{fg} in practice.
Before we conclude, we provide the results of our experimental evaluation, verifying the practical efficiency of \ac{decor}.

\section{Preliminaries} \label{sec:com_prelim}
We first introduce \acp{fg} as undirected propositional probabilistic models and afterwards define commutative factors formally.
An \ac{fg} compactly encodes a full joint probability distribution between \acp{rv}, where the full joint probability distribution is represented as a product of factors~\citep{Frey1997a,Kschischang2001a}.
\begin{definition}[Factor Graph, \citealp{Kschischang2001a}]
	An \emph{\ac{fg}} $G = (\boldsymbol V, \boldsymbol E)$ is an undirected bipartite graph consisting of a node set $\boldsymbol V = \boldsymbol R \cup \boldsymbol \Phi$, where $\boldsymbol R = \{R_1, \ldots, R_n\}$ is a set of variable nodes (\acp{rv}) and $\boldsymbol \Phi = \{\phi_1, \ldots, \phi_m\}$ is a set of factor nodes (functions), as well as a set of edges $\boldsymbol E \subseteq \boldsymbol R \times \boldsymbol \Phi$.
	The term $\range{R_i}$ denotes the possible values of a \ac{rv} $R_i$.
	There is an edge between a variable node $R_i$ and a factor node $\phi_j$ in $\boldsymbol E$ if $R_i$ appears in the argument list of $\phi_j$.
	A factor is a function that maps its arguments to a positive real number, called potential.
	The semantics of $G$ is given by
	$$
		P_G = \frac{1}{Z} \prod_{j=1}^m \phi_j(\mathcal A_j)
	$$
	with $Z$ being the normalisation constant and $\mathcal A_j$ denoting the \acp{rv} connected to $\phi_j$.
\end{definition}
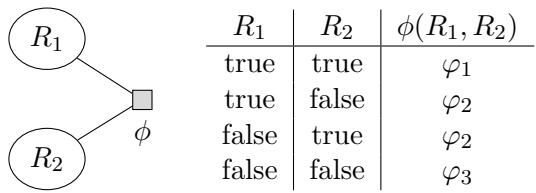
\begin{figure}
	\centering
	\input{files/example_fg.tex}
	\caption{A toy example for an \ac{fg} consisting of two Boolean \acp{rv} $R_1$ and $R_2$ as well as one factor $\phi(R_1, R_2)$. The input-output pairs of $\phi$ are given in the table on the right, where each assignment of $\phi$'s arguments is mapped to a potential with $\varphi_1$, \ldots, $\varphi_3 \in \mathbb{R}^+$.}
	\label{fig:example_fg}
\end{figure}
\begin{example}
	\Cref{fig:example_fg} displays a toy example for an \ac{fg} consisting of two \acp{rv} $R_1$ and $R_2$ as well as one factor $\phi(R_1, R_2)$.
	Both $R_1$ and $R_2$ are Boolean, that is, $\range{R_1} = \range{R_2} = \{\mathrm{true}, \mathrm{false}\}$.
	The input-output pairs of $\phi$ are specified in the table on the right.
	Note that, in this specific example, it holds that $\phi(\mathrm{true}, \mathrm{false}) = \phi(\mathrm{false}, \mathrm{true}) = \varphi_2$.
\end{example}
Lifted inference algorithms exploit symmetries in an \ac{fg} by operating on lifted representations (e.g., \acp{pfg}), which consist of \aclp{prv} and \aclp{pf}, representing sets of \acp{rv} and factors, respectively~\citep{Poole2003a}.
Symmetries in \acp{fg} are highly relevant in many real world domains such as an epidemic domain, where each person influences the probability of an epidemic in the same way---because the probability of having an epidemic depends on the number of sick people and not on individual people being sick.
The probability for an epidemic is the same if there are three sick people and the remaining people in the universe are not sick, independent of whether $alice$, $bob$, and $eve$ or $charlie$, $dave$, and $fred$ are sick.
Analogously, for movies the popularity of an actor influences the success of a movie in the same way for each actor being part of the movie. 

To detect symmetries in an \ac{fg} and obtain a lifted representation to speed up inference, the so-called \ac{cpr} algorithm~\citep{Luttermann2024a} can be applied.
The \ac{cpr} algorithm builds on the \acl{cp} algorithm~\citep{Kersting2009a,Ahmadi2013a} and transforms a given \ac{fg} into a \ac{pfg} entailing equivalent semantics as the initial \ac{fg}.
We provide a formal description of the \ac{cpr} algorithm in \cref{appendix:com_acp}.
During the course of the algorithm, \ac{cpr} checks for commutative factors, which play a crucial role in symmetry detection and are defined as follows.
\begin{definition}[Commutative Factor, \citealp{Luttermann2024a}] \label[definition]{def:com_commutative}
	Let $\phi(R_1, \ldots, R_n)$ denote a factor.
	We say that $\phi$ is \emph{commutative with respect to} $\boldsymbol S \subseteq \{R_1, \ldots, R_n\}$ if for all events $r_1, \ldots, r_n \in \times_{i=1}^n \range{R_i}$ it holds that $\phi(r_1, \ldots, r_n) = \phi(r_{\pi(1)}, \ldots, r_{\pi(n)})$ for all permutations $\pi$ of $\{1, \ldots, n\}$ with $\pi(i) = i$ for all $R_i \notin \boldsymbol S$.
	If $\phi$ is commutative with respect to $\boldsymbol S$, we say that all arguments in $\boldsymbol S$ are \emph{commutative arguments} of $\phi$.
\end{definition}
\begin{example}
	Consider again the factor $\phi(R_1,R_2)$ depicted in \cref{fig:example_fg}.
	Since it holds that $\phi(\mathrm{true}, \mathrm{false}) = \phi(\mathrm{false}, \mathrm{true}) = \varphi_2$, $\phi$ is commutative with respect to $\{R_1, R_2\}$.
\end{example}
By definition, every factor is commutative with respect to any singleton subset of its arguments independent of the factor's potential mappings.
Note that, in practice, we are only interested in factors that are commutative with respect to a subset $\boldsymbol S$ of their arguments where $\abs{\boldsymbol S} > 1$ because our intention is to group indistinguishable arguments (and grouping a single element does not yield any benefit).
Since all arguments in $\boldsymbol S$ are candidates to be grouped together, the task we solve in this paper is to compute a subset $\boldsymbol S$ of a factor's arguments of \emph{maximum size} such that the factor is commutative with respect to $\boldsymbol S$.
In many practical settings, factors are commutative with respect to a subset of their arguments, for example if individuals are indistinguishable and only the number of individuals having a certain property is of interest (as in an epidemic domain where the number of persons being sick determines the probability of an epidemic while it does not matter which specific persons are sick, and so on).
We next show how commutative factors can efficiently be detected in an \ac{fg}, which is crucial to exploit symmetries in the \ac{fg} for lifted inference.

\section{Efficient Detection of Commutative Factors Using Buckets} \label{sec:com_thms}
Detecting commutative factors is a fundamental part of lifted model construction.
The current state of the art, however, applies a rather \enquote{naive} approach to compute a subset of commutative arguments of maximum size.
In particular, computing a maximum sized subset of commutative arguments of a factor $\phi$ is currently implemented by iterating over all possible subsets of $\phi$'s arguments in order of descending size of the subsets.
That is, for a factor $\phi(R_1, \ldots, R_n)$, it is first checked whether $\phi$ is commutative with respect to all of its $n$ arguments, then it is checked whether $\phi$ is commutative with respect to any subset consisting of $n-1$ arguments, and so on.
In the worst case, the algorithm needs $O(2^n)$ iterations to compute a maximum sized subset of commutative arguments.
Even though we have to consider $O(2^n)$ subsets in the worst case, we are able to drastically prune the search space in many practical settings, as we show next.
The idea is that we can partition the potential values a factor maps its arguments to into so-called \emph{buckets}, which allow us to restrict the space of possible candidate subsets and thereby heavily reduce the number of necessary iterations.
A bucket counts the occurrences of specific range values in an assignment for a subset of a factor's arguments.
Consequently, each bucket may contain multiple potential values and every potential value is part of exactly one bucket.
\begin{definition}[Bucket, \citealp{Luttermann2024d}]
	Let $\phi(R_1, \ldots, R_n)$ denote a factor and let $\boldsymbol S \subseteq \{R_1, \ldots, R_n\}$ denote a subset of $\phi$'s arguments such that $\range{R_i} = \range{R_j}$ holds for all $R_i, R_j \in \boldsymbol S$.
	Further, let $\mathcal V$ denote the range of the elements in $\boldsymbol S$ (identical for all $R_i \in \boldsymbol S$).
	Then, a \emph{bucket} $b$ entailed by $\boldsymbol S$ is a set of tuples $\{(v_i, n_i)\}_{i = 1}^{\abs{\mathcal V}}$, $v_i \in \mathcal V$, $n_i \in \mathbb{N}$, and $\sum_i n_i = \abs{\boldsymbol S}$, such that $n_i$ specifies the number of occurrences of value $v_i$ in an assignment for all \acp{rv} in $\boldsymbol S$.
	A shorthand notation for $\{(v_i, n_i)\}_{i = 1}^{\abs{\mathcal V}}$ is $[n_1, \dots, n_{\abs{\mathcal V}}]$.
	In abuse of notation, we denote by $\phi(b)$ the multiset of potentials the assignments represented by $b$ are mapped to by $\phi$.
	The set of all buckets entailed by $\phi$ is denoted as $\mathcal B(\phi)$.
\end{definition}
\begin{table}
	\centering
	\input{files/example_buckets.tex}
	\caption{A factor $\phi(R_1, R_2, R_3)$ with three Boolean arguments and the corresponding mappings of potential values (left table) as well as the partition of potential values into buckets (right table). Here, $\phi$ is commutative with respect to $\boldsymbol S = \{R_2, R_3\}$.}
	\label{fig:example_buckets}
\end{table}
\begin{example}
	Consider the factor $\phi(R_1, R_2, R_3)$ depicted in \cref{fig:example_buckets} and let $\boldsymbol S = \{R_1, R_2, R_3\}$ with $\range{R_1} = \range{R_2} = \range{R_3} = \{\mathrm{true}, \mathrm{false}\}$.
	Then, $\boldsymbol S$ entails the four buckets $\{(\mathrm{true}, 3), (\mathrm{false}, 0)\}$, $\{(\mathrm{true}, 2), (\mathrm{false}, 1)\}$, $\{(\mathrm{true}, 1), (\mathrm{false}, 2)\}$, and $\{(\mathrm{true}, 0), (\mathrm{false}, 3)\}$ (or $[3,0]$, $[2,1]$, $[1,2]$, and $[0,3]$, respectively, in shorthand notation).
	Following the mappings given in the table from \cref{fig:example_buckets}, it holds that $\phi([3,0]) = \langle \varphi_1 \rangle$, $\phi([2,1]) = \langle \varphi_2, \varphi_2, \varphi_4 \rangle$, $\phi([1,2]) = \langle \varphi_3, \varphi_5, \varphi_5 \rangle$, and $\phi([0,3]) = \langle \varphi_6 \rangle$.
\end{example}
Buckets allow us to restrict the candidates of arguments that are possibly commutative (exchangeable) and hence might be grouped.
The idea is that commutative arguments can be replaced by a bucket that counts over their range values instead of listing each of the arguments separately.
In particular, each combination of bucket for the commutative arguments and fixed values for the remaining arguments must be mapped to the same potential value, as illustrated in the upcoming example.
\begin{table}
	\centering
	\input{files/example_buckets_crv.tex}
	\caption{A compressed representation of the factor $\phi$ given in \cref{fig:example_buckets}. The \acp{rv} $R_2$ and $R_3$ are now replaced by a so-called \acl{crv} $\#_X[R(X)]$, which uses buckets to represent $R_2$ and $R_3$ simultaneously. Note that the semantics remains unchanged while the size of the table got reduced due to grouping commutative (exchangeable) \acp{rv}.}
	\label{fig:example_buckets_crv}
\end{table}
\begin{example}
	Take a look at \cref{fig:example_buckets_crv}, which contains a compressed representation of the factor $\phi$ from \cref{fig:example_buckets}.
	Originally, $\phi$ maps each combination of bucket for $R_2$ and $R_3$ and fixed values for $R_1$ to the same potential value, e.g., $R_1 = \mathrm{true}$, $R_2 = \mathrm{true}$, $R_3 = \mathrm{false}$ and $R_1 = \mathrm{true}$, $R_2 = \mathrm{false}$, $R_3 = \mathrm{true}$ (that is, the combination of value $\mathrm{true}$ for $R_1$ and bucket $[1,1]$ for $R_2$ and $R_3$) are mapped to the same potential value $\varphi_2$.
	Consequently, these two assignments can be represented by a single assignment $(\mathrm{true}, [1,1])$ in the compressed representation shown in \cref{fig:example_buckets_crv}.
	In general, it is possible to compress all commutative arguments in a factor by replacing them by a so-called \acl{crv}, which uses buckets as a representation.
	In this particular example, $\phi$ is commutative with respect to $\boldsymbol S = \{R_2, R_3\}$ and thus, we are able to group $R_2$ and $R_3$ together.
\end{example}
Grouping commutative arguments does not alter the semantics of the factor and at the same time reduces the size of the table that has to be stored to encode the input-output mappings of the factor as well as the inference time.
To obtain the maximum possible compression, we therefore aim to compute a \emph{maximum sized} subset of commutative arguments.

A crucial observation to efficiently compute a maximum sized subset of commutative arguments is that the buckets of the initially given factor must contain duplicate potential values as a necessary condition for arguments to be able to be commutative.
In particular, groups of identical potential values within a bucket directly restrict the possible candidates of commutative arguments, as we show next.
\begin{theorem} \label{th:com_args_by_buckets}
	Let $\phi(R_1, \ldots, R_n)$ be a factor and let $b \in \mathcal B(\phi)$ with $\abs{\phi(b)} > 1$ be a bucket entailed by $\phi$.
	Further, let $G = \{\varphi_1, \ldots, \varphi_{\ell}\}$ with $\abs{G} > 2$ denote an arbitrary maximal set of identical potential values in $\phi(b)$ and let $(r_1^1, \ldots, r_n^1), \ldots, (r_1^{\ell}, \ldots, r_n^{\ell})$ denote all assignments corresponding to the potential values in $G$.
	Then, a subset $\boldsymbol S \subseteq \{R_1, \ldots, R_n\}$ with $\abs{\boldsymbol S} > 2$ such that $\phi$ is commutative with respect to $\boldsymbol S$ is obtained by computing the element-wise intersection $(r_1^1, \ldots, r_n^1) \cap (r_1^{\ell}, \ldots, r_n^{\ell})$ and adding all arguments $R_i$ to $\boldsymbol S$ for which it holds that $\{r_i^1\} \cap \ldots \cap \{r_i^{\ell}\} = \emptyset$.
	Here, the element-wise intersection of two assignments $(x_1, \ldots, x_{\ell})$ and $(y_1, \ldots, y_{\ell})$ contains the value $x_i$ at position $i$ if $x_i = y_i$, otherwise position $i$ equals $\emptyset$.
\end{theorem}
\begin{proof}
	It holds that $\phi(r_1^1, \ldots, r_n^1) = \varphi_1, \ldots, \phi(r_1^{\ell}, \ldots, r_n^{\ell}) = \varphi_{\ell}$ and $\varphi_1 = \ldots = \varphi_{\ell}$.
	Moreover, let $\boldsymbol S \subseteq \{R_1, \ldots, R_n\}$ denote any subset of commutative arguments such that $\abs{\boldsymbol S} > 2$.
	Recall that according to \cref{def:com_commutative}, for all events $r_1, \ldots, r_n \in \times_{i=1}^n \range{R_i}$ it holds that $\phi(r_1, \ldots, r_n) = \phi(r_{\pi(1)}, \ldots, r_{\pi(n)})$ for all permutations $\pi$ of $\{1, \ldots, n\}$ with $\pi(i) = i$ for all $R_i \notin \boldsymbol S$.
	Thus, the positions of all arguments which are not in $\boldsymbol S$ are fixed in all permutations.
	Consequently, the element-wise intersection of the assignments yields non-empty sets for all positions belonging to arguments not in $\boldsymbol S$ because these assignments contain identical assigned values at those positions.
	At the same time, as we consider only buckets with at least two elements, all assignments contain at least two distinct assigned values because all assignments that contain only a single value are those that map to a bucket containing a single element.
	We further consider all assignments $(r_{\pi(1)}, \ldots, r_{\pi(n)})$ for all permutations $\pi$ of $\{1, \ldots, n\}$ and therefore, there are at least two different assigned values for all positions belonging to arguments in $\boldsymbol S$, yielding an empty set when computing the element-wise intersection of the assignments for a position of any argument in $\boldsymbol S$.
\end{proof}
\Cref{th:com_args_by_buckets} tells us that candidate subsets of commutative arguments can be found by first identifying groups of identical potential values within a bucket and then computing the element-wise intersection of the assignments corresponding to these potential values.
The following example illustrates how \cref{th:com_args_by_buckets} can be applied to identify commutative arguments.
\begin{example}
	Consider again the factor $\phi(R_1, R_2, R_3)$ depicted in \cref{fig:example_buckets}.
	For the sake of the example, let us take a look at the bucket $[2,1]$.
	We have two groups of identical values, namely $G_1 = \{\varphi_2, \varphi_2\}$ and $G_2 = \{\varphi_4\}$.
	Since $G_2$ contains only a single potential value, there are no candidates of commutative arguments induced by $G_2$.
	However, $G_1$ contains the potential value $\varphi_2$ two times and the corresponding assignments are $(\mathrm{true},\mathrm{true},\mathrm{false})$ and $(\mathrm{true},\mathrm{false},\mathrm{true})$.
	The element-wise intersection of those assignments is then given by $(\mathrm{true},\mathrm{true},\mathrm{false}) \cap (\mathrm{true},\mathrm{false},\mathrm{true}) = (\mathrm{true},\emptyset,\emptyset)$.
	As the element-wise intersection is empty at positions two and three, the set $\{R_2, R_3\}$ is a possible subset of commutative arguments.
\end{example}
By looking for \emph{maximal} sets $G$ of identical potential values, we ensure that we actually find maximum sized candidate subsets of commutative arguments.
Moreover, the number of positions having an empty element-wise intersection directly corresponds to the size of every candidate subset $\boldsymbol S$, that is, the number of positions having an empty element-wise intersection is equal to $\abs{\boldsymbol S}$.
As all permutations of these $\abs{\boldsymbol S}$ positions are mapped to the same potential value, we know that if there is a subset of commutative arguments of size $\abs{\boldsymbol S}$, there must be at least $\abs{\boldsymbol S}$ duplicate potential values in every bucket of size at least two.
\begin{corollary} \label[corollary]{th:duplicates_in_buckets}
	Let $\phi(R_1, \ldots, R_n)$ be a factor and let $\boldsymbol S \subseteq \{R_1, \ldots, R_n\}$ be a subset of arguments such that $\phi$ is commutative with respect to $\boldsymbol S$.
	Then, in every bucket $b \in \mathcal B(\phi)$ with $\abs{\phi(b)} > 1$, there exists a potential value $\varphi$ that occurs at least $\abs{\boldsymbol S}$ times in $\phi(b)$.
\end{corollary}
\begin{example}
	Consider again the factor $\phi(R_1, R_2, R_3)$, which is commutative with respect to $\boldsymbol S = \{R_2, R_3\}$, depicted in \cref{fig:example_buckets}.
	Therefore, $\phi$ must map the buckets $[2,1]$ and $[1,2]$ (the only buckets that are mapped to at least two elements by $\phi$) to at least $\abs{\boldsymbol S} = 2$ identical potential values.
	Here, $\varphi_2$ occurs twice in $\phi([2,1])$ and $\varphi_5$ occurs twice in $\phi([1,2])$.
\end{example}
As we know for every bucket $b$ containing more than one element that there must exist a potential occurring at least $\abs{\boldsymbol S}$ times in $b$ to allow for a subset $\boldsymbol S$ of commutative arguments to exist, we also know that there cannot be more commutative arguments than the minimum number of duplicate potential values over all buckets containing more than one element.
\begin{corollary} \label[corollary]{th:com_args_bound}
	Let $\phi(R_1, \ldots, R_n)$ denote a factor.
	Then, the size of any subset of commutative arguments of $\phi$ is upper-bounded by
	$$
		\min_{b \in \{b \mid b \in \mathcal B(\phi) \land \abs{\phi(b)} > 1\}} \max_{\varphi \in \phi(b)} \cnt(\phi(b), \varphi),
	$$
	where $\cnt(\phi(b), \varphi)$ denotes the number of occurrences of potential $\varphi$ in $\phi(b)$.
\end{corollary}
\begin{example}
	The factor $\phi(R_1, R_2, R_3)$ given in \cref{fig:example_buckets} maps $[2,1]$ to $\langle \varphi_2, \varphi_2, \varphi_4 \rangle$ and $[1,2]$ to $\langle \varphi_3, \varphi_5, \varphi_5 \rangle$ and thus, the upper bound for the size of any subset of commutative arguments of $\phi$ is two because $\phi$ maps the buckets $[2,1]$ and $[1,2]$ to two identical potential values each.
\end{example}
The bounds implied by \cref{th:duplicates_in_buckets,th:com_args_bound} show that the number of possible candidate subsets of commutative arguments can be heavily reduced in many practical settings.
In particular, we are now able to exploit the insights from \cref{th:com_args_by_buckets} as well as \cref{th:duplicates_in_buckets,th:com_args_bound} to drastically restrict the search space of possible candidate subsets of commutative arguments, thereby avoiding the naive iteration over all $O(2^n)$ subsets of a factor's arguments and thus making the check for commutativity feasible in practice.

We next incorporate our theoretical findings into a practical algorithm, called \ac{decor}, to efficiently compute subsets of commutative arguments of maximum size in practice.

\section{The \acs{decor} Algorithm} \label{sec:com_decor}
We are now ready to gather the theoretical results to obtain the \ac{decor} algorithm, which efficiently computes a maximum sized subset of commutative arguments of a given factor.
\Cref{alg:decor} presents the whole \ac{decor} algorithm, which undertakes the following steps on a given input factor $\phi(R_1, \dots, R_n)$ to compute subsets of commutative arguments.

\begin{algorithm2e}[t]
	\caption{Detection of Commutative Factors (DECOR)}
	\label{alg:decor}
	\DontPrintSemicolon
	\LinesNumbered
	\KwIn{A factor $\phi(R_1, \dots, R_n)$.}
	\KwOut{All subsets $\boldsymbol S \subseteq \{R_1, \dots, R_n\}$ with $\abs{\boldsymbol S} \geq 2$ such that $\phi$ is commutative with respect to $\boldsymbol S$, or $\emptyset$ if no such subset $\boldsymbol S$ exists.}
	\BlankLine
	$C \gets \{\{R_1, \dots, R_n\}\}$\;
	\ForEach{$b \in \mathcal B(\phi)$}{
		\If{$\abs{\phi(b)} < 2$}{
			\Continue\;
		}
		$G_1, \dots, G_{\ell} \gets$ Partition $\phi(b)$ into maximal groups of identical values such that $\abs{G_i} \geq 2$ for all $i \in \{1, \dots, \ell\}$\; \label{line:partition_max_groups}
		\If{$\ell < 1$}{
			\Return{$\emptyset$}\;
		}
		$C' \gets \emptyset$\;
		\ForEach{$G_i \in \{G_1, \dots, G_{\ell}\}$}{ \label{line:loop_groups}
			$P_1, \dots, P_n \gets$ Intersection of positions corresponding to all potentials in $G_i$\; \label{line:pos_intersection}
			$C_i \gets \{R_i \mid i \in \{1, \dots, n\} \land P_i = \emptyset\}$\;
			\If{$\nexists C_j \in C': C_i \subseteq C_j$}{
				$C' \gets C' \cup \{C_i\}$\;
			}
		}
		$C_{\cap} \gets \emptyset$\;
		\ForEach{$C_i \in C$}{ \label{line:loop_intersect_outer}
			\ForEach{$C_j \in C'$}{ \label{line:loop_intersect_inner}
				\If{$\abs{C_i \cap C_j} \geq 2 \land \nexists C'' \in C_{\cap}: C_i \cap C_j \subseteq C''$}{
					$C_{\cap} \gets C_{\cap} \cup \{C_i \cap C_j\}$\;
				}
			}
		}
		$C \gets C_{\cap}$\;
		\If{$C = \emptyset$}{
			\Return{$\emptyset$}\;
		}
	}
	\Return{$C$}\;
\end{algorithm2e}

First, \ac{decor} initialises a set $C$ of possible candidate subsets of commutative arguments, which contains the whole argument list of $\phi$ at the beginning.
\Ac{decor} then iterates over all buckets entailed by $\phi$, where all buckets containing only a single element are skipped because they do not restrict the search space.
\Ac{decor} then partitions the values in every bucket into maximal groups $G_1, \ldots, G_{\ell}$ of identical potential values, where each group must contain at least two identical potential values as we are only interested in subsets of commutative arguments of size at least two.
If no such group is found, \ac{decor} returns an empty set as there are no candidates for a subset of commutative arguments of size at least two.
Afterwards, \ac{decor} applies the insights from \cref{th:com_args_by_buckets} and iterates over all groups of identical potential values to compute the element-wise intersection of the assignments corresponding to the potential values within a group.
For every group $G_i$ of identical potential values, \ac{decor} then builds a set $C_i$ containing all arguments whose position in the element-wise intersection over their corresponding assignments is empty.
\Ac{decor} then adds the set $C_i$ of arguments that are commutative according to the current bucket to the set $C'$ of candidate subsets for the current bucket if $C_i$ is not subsumed by any other candidate subset in $C'$.
Thereafter, as each bucket further restricts the search space for possible commutative arguments, \ac{decor} computes the intersection of all candidate subsets collected so far in $C$ and all candidate subsets of the current bucket in $C'$.
\Ac{decor} keeps all candidate subsets of size at least two after the intersection that are not already subsumed by another candidate subset after the intersection.
In case there are no candidate subsets of size at least two left after considering a specific bucket, \ac{decor} returns an empty set as a result.
Finally, if there are still candidate subsets left after iterating over all buckets of size at least two, \ac{decor} returns the set of all candidate subsets of commutative arguments of maximum size.
A maximum sized subset of commutative arguments can then be obtained by selecting any candidate subset in $C$ that contains the most elements.

Note that \ac{decor} returns \emph{all} candidate subsets of commutative arguments that are not subsumed by another candidate subset and could be replaced by a so-called \acl{crv} in the lifted representation of the factor.
As we aim for the most compression possible, we are mostly interested in the maximum sized subset of commutative arguments.
However, it might be conceivable that there are settings in which counting over two or more candidate subsets of commutative arguments yields a higher compression than counting over a single maximum sized subset.
We therefore keep \ac{decor} as a general algorithm that is able to return all candidate subsets of commutative arguments that are not subsumed by another candidate subset and leave the decision of which subset(s) to choose to the user.
\begin{example}
	Let us take a look at how \ac{decor} computes a maximum sized subset of commutative arguments for the factor $\phi(R_1, R_2, R_3)$ from \cref{fig:example_buckets}.
	Initially, \ac{decor} starts with the set $C = \{\{R_1, R_2, R_3\}\}$.
	After skipping the bucket $[3,0]$, \ac{decor} finds the group $G_1 = \{\varphi_2, \varphi_2\}$ for the bucket $[2,1]$.
	Other groups contain less than two elements and are thus ignored.
	\Ac{decor} then computes the element-wise intersection of the assignments corresponding to the potential values in $G_1$, which is given by $(\mathrm{true},\mathrm{true},\mathrm{false}) \cap (\mathrm{true},\mathrm{false},\mathrm{true}) = (\mathrm{true},\emptyset,\emptyset)$.
	As the element-wise intersection is empty at positions two and three, \ac{decor} adds the set $C_i = \{R_2, R_3\}$ to the set $C'$ of candidate subsets for the bucket $[2,1]$.
	Afterwards, \ac{decor} computes the intersection of $\{R_1, R_2, R_3\}$ (as $C = \{\{R_1, R_2, R_3\}\}$) and $\{R_2, R_3\}$ (as $C' = \{\{R_2, R_3\}\}$), which yields $C_{\cap} = \{\{R_2, R_3\}\}$ and thus also $C = \{\{R_2, R_3\}\}$ for the next iteration.
	\Ac{decor} next finds the group $G_1 = \{\varphi_5, \varphi_5\}$ for the bucket $[1,2]$ and computes the element-wise intersection of the assignments corresponding to the potential values in $G_1$, which is given by $(\mathrm{false},\mathrm{true},\mathrm{false}) \cap (\mathrm{false},\mathrm{false},\mathrm{true}) = (\mathrm{false},\emptyset,\emptyset)$.
	Again, positions two and three are empty and thus, \ac{decor} adds the set $C_i = \{R_2, R_3\}$ to the set $C'$ of candidate subsets for the bucket $[1,2]$.
	Thereafter, \ac{decor} computes the intersection of $\{R_2, R_3\}$ (as $C = \{\{R_2, R_3\}\}$) and $\{R_2, R_3\}$ (as $C' = \{\{R_2, R_3\}\}$), which yields $C = C_{\cap} = \{\{R_2, R_3\}\}$.
	Finally, as the next bucket $[0,3]$ is skipped again and there are no other buckets left, \ac{decor} returns the set $\{\{R_2, R_3\}\}$ containing a single candidate subset of commutative arguments, which is at the same time a maximum sized subset.
\end{example}
We remark that the number of candidate subsets considered by \ac{decor} now directly depends on the number of groups of identical potential values in the buckets of the factor.
In most practical settings, potential values are not identical unless the factor is commutative with respect to a subset of its arguments.
Therefore, \ac{decor} heavily restricts the search space for possible candidate subsets of commutative arguments in most practical settings.

Recall that the \enquote{naive} algorithm iterates over $O(2^n)$ subsets of arguments and then has to consider all $\abs{\mathcal R}^n$ (where $\mathcal R$ denotes the range of the arguments and $n$ is the total number of arguments) potential values in the table of potential mappings during each iteration to check whether a subset of arguments is commutative.
While \ac{decor} also has to consider all $\abs{\mathcal R}^n$ potential values during the course of the algorithm (in fact, every algorithm searching for commutative arguments has to take a look at every potential value at least once), \ac{decor} considers every potential value just twice (in \cref{line:partition_max_groups,line:pos_intersection}).
\begin{theorem}
	Let $\phi(R_1, \ldots, R_n)$ denote a factor and let $\mathcal R = \range{R_1} = \ldots = \range{R_n}$.
	The worst-case time complexity of the \enquote{naive} algorithm to compute a maximum sized subset of commutative arguments is in $O(2^n \cdot \abs{\mathcal R}^n \cdot n)$.
\end{theorem}
\begin{proof}
	The \enquote{naive} algorithm iterates over all subsets of $\phi$'s arguments and then iterates for every subset over all buckets induced by that subset to check whether each bucket is mapped to a unique potential value by $\phi$.
	To compute the buckets, all $n$ positions in all $\abs{\mathcal R}^n$ assignments are considered.
	The iteration over all values in the buckets requires time $O(\abs{\mathcal R}^n)$, as every potential value is looked at once.
	In the worst case, the \enquote{naive} algorithm iterates over $O(2^n)$ subsets of arguments and thus has a complexity of $O(2^n \cdot \abs{\mathcal R}^n \cdot n)$.
\end{proof}
\begin{theorem}
	Let $\phi(R_1, \ldots, R_n)$ denote a factor and let $\mathcal R = \range{R_1} = \ldots = \range{R_n}$.
	The expected worst-case time complexity of the \ac{decor} algorithm to compute a maximum sized subset of commutative arguments is in $O(\abs{\mathcal B(\phi)} \cdot k \cdot \binom{n}{n/2} + \abs{\mathcal R}^n \cdot n)$, where $k$ denotes the maximum number of potential values within a bucket entailed by $\phi$.
\end{theorem}
\begin{proof}
	The partitioning of $\phi(b)$ into maximal groups of identical potential values (\cref{line:partition_max_groups}) requires time linear in the number of potential values in the bucket $b$ and as all $\abs{\mathcal R}^n$ potential values are looked at exactly once over all buckets, the partitioning requires time $O(\abs{\mathcal R}^n)$ in total for all buckets.
	\Ac{decor} finds at most $\lfloor\abs{\mathcal R}^n/2\rfloor$ groups of identical potential values (if every group contains exactly two identical potential values) over all buckets.
	The loop in \cref{line:loop_groups} thus needs $O(\abs{\mathcal R}^n)$ iterations in the worst case for all buckets entailed by $\phi$.
	Computing the intersection over all positions is done in time $O(n)$ and subset testing can be done in expected time $O(n)$, thereby yielding an expected time complexity of $O(\abs{\mathcal R}^n \cdot n)$ for the loop in \cref{line:loop_groups}.
	For each bucket, \ac{decor} adds at most $k/2$ candidate subsets to $C'$ and as there are at most $\binom{n}{n/2}$ candidate subsets in $C$ (there are $\binom{n}{j}$ subsets of size $j$ and the binomial coefficient reaches its maximum at $j = n/2$), the loops in \cref{line:loop_intersect_outer,line:loop_intersect_inner} run in time $O(k \cdot \binom{n}{n/2})$ per bucket, yielding a total run time of $O(\abs{\mathcal B(\phi)} \cdot k \cdot \binom{n}{n/2} + \abs{\mathcal R}^n \cdot n)$.
\end{proof}
There are $\abs{\mathcal B(\phi)} = \binom{n+\abs{\mathcal R}-1}{n}$ buckets for a factor $\phi(R_1, \ldots, R_n)$ where $\mathcal R = \range{R_1} = \ldots = \range{R_n}$\footnote{The number of buckets is given by the number of weak compositions of $n$ into $\abs{\mathcal R}$ parts.}.
For example, a factor with $n$ Boolean arguments ($\abs{\mathcal R} = 2$) entails $\binom{n + 1}{n} = n+1$ buckets.
Of these buckets, \ac{decor} only has to consider $n-1$ buckets, as two map to one potential.  
Note that the worst-case complexity of \ac{decor} is overly pessimistic as it is impossible to have the maximum number of groups $G_1, \ldots, G_{\ell}$ and the maximum number of candidate subsets $C$ simultaneously.
The binomial coefficient is a generous upper bound nearly impossible to reach in practice, which is also confirmed in our experiments.

Further, we argue that $n/2$, instead of the binomial coefficient, is still a generous upper bound in practice for the maximum number of groups and the maximum number of candidate subsets.
In case there are hardly any symmetries within the factor, there are also hardly any duplicate potentials in a bucket.
With hardly any duplicate potentials in a bucket, the number of candidate sets is small.
In case the model is symmetric in general and there are symmetries within a factor, then there is a good chance that in each bucket at least one potential occurs quite often. 
This case leads to few but large candidate sets.
Further, the candidate sets across all buckets are likely to be rather identical.
Thus, having $n/2$ as an upper bound for both the maximum number of groups and the maximum number of candidate subsets is still pessimistic in practice, as \ac{decor} most likely either finds few (larger) sets or hardly any sets at all. 
Overall, our arguments why \ac{decor} finishes fast in practice are:
In case the model is highly symmetric, symmetries within a factor are expected to be over larger sets.
Thus, \ac{decor} has just few sets to consider.
In case there are hardly any symmetries within a factor and the model itself, \ac{decor} has to check few small candidate sets and might even end up with an empty candidate set after few buckets.

Finally, before we demonstrate the practical efficiency of \ac{decor} in our experimental evaluation, we remark that \ac{decor} can handle factors with arguments having arbitrary ranges.
During the course of this paper, we consider identical ranges of all arguments for brevity, however, it is possible to apply \ac{decor} to factors with arguments having various ranges $\mathcal R_1, \ldots, \mathcal R_k$ by considering the buckets for all arguments with range $\mathcal R_i$ separately for all $i \in \{1, \ldots, k\}$ as only arguments having the same range can be commutative at all.

\section{Experiments} \label{sec:com_eval}
In addition to our theoretical results, we conduct an experimental evaluation of \ac{decor} to assess its performance in practice.
We compare the run times of \ac{decor} to the \enquote{naive} approach, i.e., iterating over all possible subsets of a factor's arguments in order of descending size of the subsets.
For our experiments, we generate factors with $n \in \{2,4,6,8,10,12,14,16\}$ Boolean arguments, of which $k \in \{0,2,\lfloor\frac{n}{2}\rfloor,n-1,n\}$ arguments are commutative.
We run each algorithm with a timeout of five minutes per instance and report the average run time over all instances for each choice of $n$.

\Cref{fig:plot-avg} displays the average run times of the algorithms on a logarithmic scale.
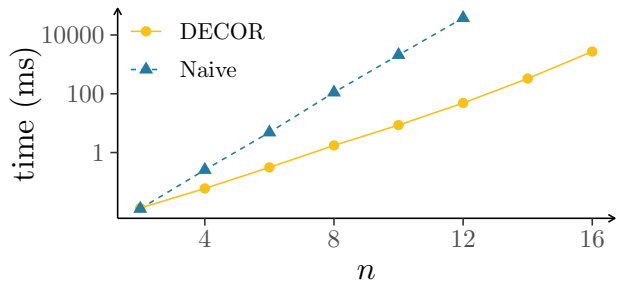
\begin{figure}
	\centering
	\input{files/plot-avg.tex}
	\caption{Average run times of \ac{decor} and the \enquote{naive} algorithm for factors with different numbers of arguments $n$. For each choice of $n$, we report the average run time over all choices of numbers of commutative arguments $k$, where $k \in \{0,2,\lfloor\frac{n}{2}\rfloor,n-1,n\}$.}
	\label{fig:plot-avg}
\end{figure}
While both \ac{decor} and the \enquote{naive} algorithm are capable of solving small instances with less than ten arguments in under a second, the \enquote{naive} algorithm struggles to solve instances with increasing $n$.
More specifically, at $n = 8$, \ac{decor} is already faster than the \enquote{naive} algorithm by a factor of 100 and the factor drastically increases with an increasing size of $n$.
After $n = 12$, the \enquote{naive} algorithm runs into a timeout, which is not surprising as it iterates over $O(2^n)$ subsets.
\Ac{decor}, on the other hand, solves all instances within the specified timeout and is able to handle large instances in an efficient manner.
We remark that the \enquote{naive} algorithm is able to solve instances with $n > 12$ if $k = n-1$ or $k = n$, which is due to its iterations being performed in order of descending size of the subsets.
However, as the \enquote{naive} algorithm runs into a timeout for all remaining choices of $k$, it is not possible to compute a meaningful average run time.
We therefore provide additional experimental results in \cref{appendix:com_further_results}, where we report run times for each choice of $k$ separately.

\section{Conclusion}
We present the \ac{decor} algorithm to efficiently detect commutative factors in \acp{fg}.
\Ac{decor} makes use of buckets to significantly restrict the search space and thereby efficiently handles factors even with a large number of arguments where previous algorithms fail to compute a solution within a timeout.
We prove that the number of subsets to check for commutative arguments is upper-bounded depending on the number of duplicate potential values in the buckets.
By exploiting this upper bound, \ac{decor} drastically reduces the number of subsets to check for commutative arguments compared to previous algorithms, which iterate over $O(2^n)$ subsets for a factor with $n$ arguments in the worst case.
Additionally, \ac{decor} returns \emph{all} candidate subsets of commutative arguments that are not subsumed by another candidate subset, allowing to compress the model even further.
Overall, \ac{decor} drastically increases the efficiency to identify commutative factors and further reduces the model.

\acks{This work is funded by the BMBF project AnoMed 16KISA057.}

\bibliography{references}

\clearpage
\appendix
\acresetall

\section{Formal Description of the Advanced Colour Passing Algorithm} \label{appendix:com_acp}
The \ac{cpr} algorithm introduced by \citet{Luttermann2024a} builds on the \acl{cp} algorithm~\citep{Kersting2009a,Ahmadi2013a} and solves the problem of constructing a lifted representation---more specifically, a so-called \ac{pfg}---from a given \ac{fg}.
The idea of \ac{cpr} is to first find symmetries in a propositional \ac{fg} and then group together symmetric subgraphs.
\Ac{cpr} looks for symmetries based on potentials of factors, on ranges and evidence of \acp{rv}, as well as on the graph structure by passing around colours.
\Cref{alg:cp_revisited} provides a formal description of the \ac{cpr} algorithm, which proceeds as follows.

\begin{algorithm2e}
	\caption{Advanced Colour Passing~\citep{Luttermann2024a}}
	\label{alg:cp_revisited}
	\DontPrintSemicolon
	\LinesNumbered
	\KwIn{An \ac{fg} $G$ with \acp{rv} $\boldsymbol R = \{R_1, \ldots, R_n\}$ and factors $\boldsymbol \Phi = \{\phi_1, \ldots, \phi_m\}$, as well as a set of evidence $\boldsymbol E = \{R_1 = r_1, \ldots, R_k = r_k\}$.}
	\KwOut{A lifted representation $G'$ in form of a \ac{pfg} entailing equivalent semantics to $G$.}
	\BlankLine
	Assign each $R_i$ a colour according to $\range{R_i}$ and $\boldsymbol E$\;
	Assign each $\phi_i$ a colour according to order-independent potentials and rearrange arguments accordingly\;
	\Repeat{grouping does not change}{
		\ForEach{factor $\phi \in \boldsymbol \Phi$}{
			$signature_{\phi} \gets [\,]$\;
			\ForEach{\ac{rv} $R \in neighbours(G, \phi)$}{
				\tcp{In order of appearance in $\phi$}
				$append(signature_{\phi}, R.colour)$\;
			}
			$append(signature_{\phi}, \phi.colour)$\;
		}
		Group together all $\phi$s with the same signature\;
		Assign each such cluster a unique colour\;
		Set $\phi.colour$ correspondingly for all $\phi$s\;
		\ForEach{\ac{rv} $R \in \boldsymbol R$}{
			$signature_{R} \gets [\,]$\;
			\ForEach{factor $\phi \in neighbours(G, R)$}{
				\If{$\phi$ is commutative w.r.t.\ $\boldsymbol S$ and $R \in \boldsymbol S$}{
					$append(signature_{R}, (\phi.colour, 0))$\;
				}
				\Else{
					$append(signature_{R}, (\phi.colour, p(R, \phi)))$\;
				}
			}
			Sort $signature_{R}$ according to colour\;
			$append(signature_{R}, R.colour)$\;
		}
		Group together all $R$s with the same signature\;
		Assign each such cluster a unique colour\;
		Set $R.colour$ correspondingly for all $R$s\;
	}
	$G' \gets$ construct \acs{pfg} from groupings\;
\end{algorithm2e}

\Ac{cpr} begins with the colour assignment to variable nodes, meaning that all \acp{rv} that have the same range and observed event are assigned the same colour.
Thereafter, \ac{cpr} assigns colours to factor nodes such that factors representing identical potentials are assigned the same colour.
Two factors represent identical potentials if there exists a rearrangement of one of the factor's arguments such that both factors have identical tables of potentials when comparing them row by row~\citep{Luttermann2024d}.

After the initial colour assignments, \ac{cpr} passes the colours around.
\Ac{cpr} first passes the colours from every variable node to its neighbouring factor nodes and afterwards, every factor node $\phi$ sends its colour plus the position $p(R, \phi)$ of $R$ in $\phi$'s argument list to all of its neighbouring variable nodes $R$.
Note that factors being commutative with respect to a subset $\boldsymbol S$ of their arguments omit the position when sending their colour to a neighbouring variable node $R \in \boldsymbol S$, and therefore, \ac{cpr} has to compute for each factor a maximum sized subset of its arguments such that the factor is commutative with respect to that subset.
In its original form, \ac{cpr} iterates over all possible subsets of arguments, in descending order of their size, to compute the maximum sized subset.

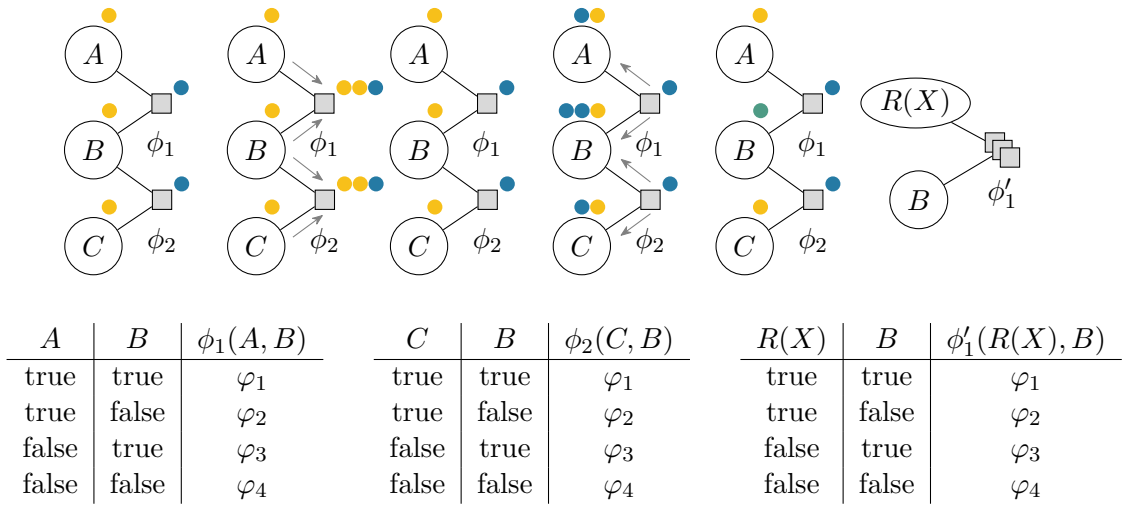
\begin{figure}[t]
	\centering
	\input{files/acp_example.tex}
	\caption{A visualisation of the steps undertaken by the \ac{cpr} algorithm on an input \ac{fg} with only Boolean \acp{rv} and no evidence (left). Colours are first passed from variable nodes to factor nodes, followed by a recolouring, and then passed back from factor nodes to variable nodes, again followed by a recolouring. The colour passing procedure is iterated until convergence and the resulting \ac{pfg} is depicted on the right. This figure is reprinted from~\protect\citep{Luttermann2024a}.}
	\label{fig:acp_example}
\end{figure}

\Cref{fig:acp_example} illustrates the \ac{cpr} algorithm on an example \ac{fg}~\citep{Ahmadi2013a}.
In this example, $A$, $B$, and $C$ are Boolean \acp{rv} with no evidence and thus, they all receive the same colour (e.g., $\mathrm{yellow}$).
As the potentials of $\phi_1$ and $\phi_2$ are identical, $\phi_1$ and $\phi_2$ are assigned the same colour as well (e.g., $\mathrm{blue}$).
The colour passing then starts from variable nodes to factor nodes, that is, $A$ and $B$ send their colour ($\mathrm{yellow}$) to $\phi_1$ and $B$ and $C$ send their colour ($\mathrm{yellow}$) to $\phi_2$.
$\phi_1$ and $\phi_2$ are then recoloured according to the colours they received from their neighbours to reduce the communication overhead.
Since $\phi_1$ and $\phi_2$ received identical colours (two times the colour $\mathrm{yellow}$), they are assigned the same colour during recolouring.
Afterwards, the colours are passed from factor nodes to variable nodes and this time not only the colours but also the position of the \acp{rv} in the argument list of the corresponding factor are shared because none of the factors is commutative with respect to a subset of its arguments having size at least two.
Consequently, $\phi_1$ sends a tuple $(\mathrm{blue}, 1)$ to $A$ and a tuple $(\mathrm{blue}, 2)$ to $B$, and $\phi_2$ sends a tuple $(\mathrm{blue}, 2)$ to $B$ and a tuple $(\mathrm{blue}, 1)$ to $C$ (positions are not shown in \cref{fig:acp_example}).
As $A$ and $C$ are both at position one in the argument list of their respective neighbouring factor, they receive identical messages and are recoloured with the same colour.
$B$ is assigned a different colour during recolouring than $A$ and $C$ because $B$ received different messages than $A$ and $C$.
The groupings do not change in further iterations and hence the algorithm terminates.
The output is the \ac{pfg} shown on the right in \cref{fig:acp_example}, where both $A$ and $C$ as well as $\phi_1$ and $\phi_2$ are grouped.

More details about the colour passing procedure and the grouping of nodes can be found in~\citep{Luttermann2024a}.
The authors also provide extensive results demonstrating the benefits of constructing and using a lifted representation for probabilistic inference.

\section{Additional Experimental Results} \label{appendix:com_further_results}
In addition to the experimental results provided in \cref{sec:com_eval}, we provide further experimental results for individual scenarios in this section.
In particular, the plot given in \cref{fig:plot-avg} displays average run times over multiple runs for different choices of the number of commutative arguments $k$.
We now showcase separate plots for each of the choices of the number of commutative arguments $k$ to highlight the influence of $k$ on the performance on the algorithms.
Again, we compare the run times of the \enquote{naive} algorithm to the run times of \ac{decor} algorithm on factors with $n \in \{2,4,6,8,10,12,14,16\}$ Boolean arguments and set a timeout of five minutes per instance.

\begin{figure}
	\centering
	\input{files/plot-k=0.tex}
	\input{files/plot-k=2.tex}
	\caption{Run times of \ac{decor} and the \enquote{naive} algorithm for factors with different numbers of arguments $n$, of which $k=0$ (left) and $k=2$ (right) arguments are commutative, respectively. Both plots use a logarithmic scale on the y-axis.}
	\label{fig:plot-k=0-k=2}
\end{figure}
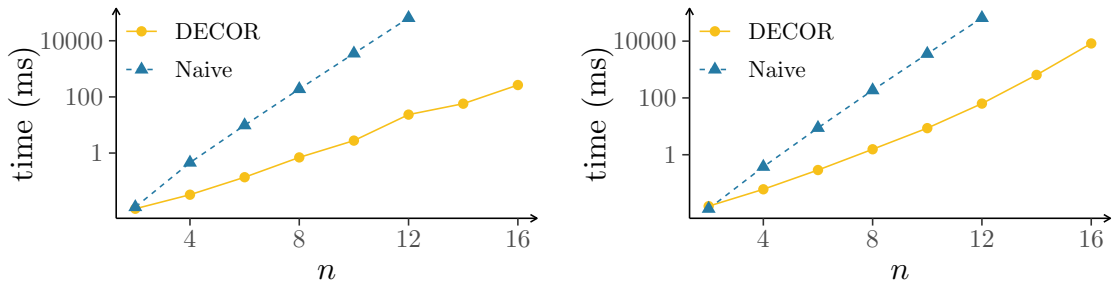

\Cref{fig:plot-k=0-k=2} shows the run times of both algorithms on factors with $k=0$ (left plot) and $k=2$ (right plot) commutative arguments on a logarithmic scale.
The performance of the \enquote{naive} algorithm is similar for $k=0$ and $k=2$.
As in \cref{fig:plot-avg}, the \enquote{naive} algorithm is able to solve instances with less than ten arguments in under a second but then faces serious scalability issues, eventually leading to timeouts for all factors with $n > 12$.
\Ac{decor} solves all instances within the specified timeout and while the run times of \ac{decor} for $k=2$ are similar to those in \cref{fig:plot-avg}, \ac{decor} is especially fast on instances where $k=0$ of the arguments are commutative.
This can be explained by the fact that all instances with $k=0$ are generated such that each potential value is unique to ensure that there are no commutative arguments.
Consequently, if there are only unique values in each bucket, \ac{decor} has no candidates to check and returns immediately (the run time still increases the larger $n$ gets because computing the buckets becomes more costly).

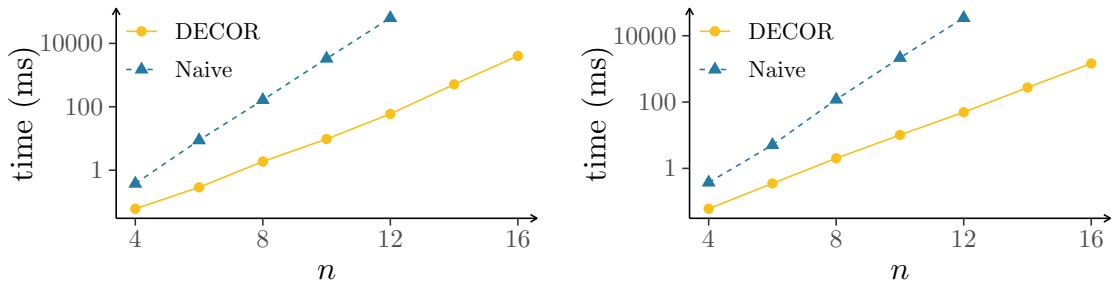
\begin{figure}
	\centering
	\input{files/plot-k=log2n.tex}
	\input{files/plot-k=ndiv2.tex}
	\caption{Run times of \ac{decor} and the \enquote{naive} algorithm for factors with different numbers of arguments $n$, of which $k=\lfloor\log_2(n)\rfloor$ (left) and $k=\lfloor\frac{n}{2}\rfloor$ (right) arguments are commutative, respectively. Both plots use a logarithmic scale on the y-axis.}
	\label{fig:plot-k=log2n-k=ndiv2}
\end{figure}
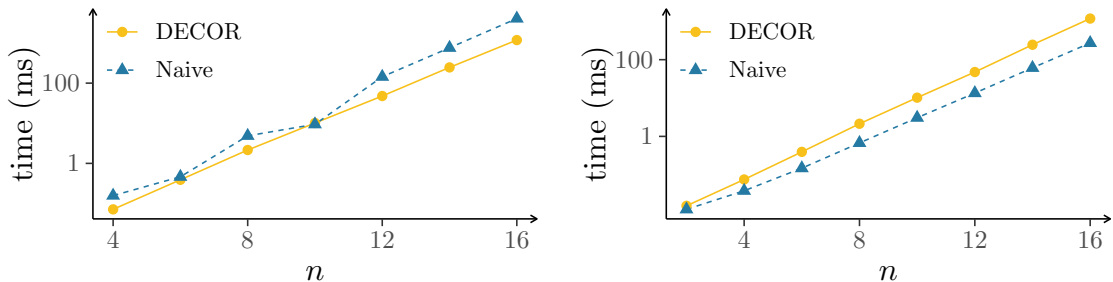
\begin{figure}
	\centering
	\input{files/plot-k=nsub1.tex}
	\input{files/plot-k=n.tex}
	\caption{Run times of \ac{decor} and the \enquote{naive} algorithm for factors with different numbers of arguments $n$, of which $k=n-1$ (left) and $k=n$ (right) arguments are commutative, respectively. Both plots use a logarithmic scale on the y-axis.}
	\label{fig:plot-k=nsub1-k=n}
\end{figure}

\Cref{fig:plot-k=log2n-k=ndiv2} presents the run times of the algorithms on factors with $k=\lfloor\log_2(n)\rfloor$ (left plot) and $k=\lfloor\frac{n}{2}\rfloor$ (right plot) commutative arguments on a logarithmic scale.
The general pattern of the run times is again similar to \cref{fig:plot-avg} and we therefore move on to take a look at \cref{fig:plot-k=nsub1-k=n}, which depicts the run times of the algorithms on factors having $k=n-1$ (left plot) and $k=n$ (right plot) commutative arguments, again on a logarithmic scale.

The left plot of \cref{fig:plot-k=nsub1-k=n} shows that the \enquote{naive} algorithm is now able to solve all instances within the specified timeout and its run times are similar to the run times of \ac{decor}.
As $k=n-1$, the \enquote{naive} algorithm is able to find a subset of commutative arguments in few iterations due to its approach of starting with the largest subsets first.
Consequently, the \enquote{naive} algorithm performs well on the right plot where $k=n$ as it considers exactly one subset and immediately finishes afterwards.
Therefore, \ac{decor} is even slightly slower than the \enquote{naive} algorithm in this particular extrem case---however, the difference is only marginal.
In conclusion, we find that the \enquote{naive} approach is generally not able to scale to numbers of arguments larger than $12$ (only in extrem cases) whereas \ac{decor} is able to handle large $n$ efficiently in every evaluated scenario.
\end{document}

%% file: defs.tex
\acrodef{bn}[BN]{Bayesian network}
\acrodef{cp}[CP]{colour passing}
\acrodef{cpr}[ACP]{advanced colour passing}
\acrodef{crv}[CRV]{counting randvar}
\acrodef{decor}[DECOR]{detection of commutative factors}
\acrodef{deft}[DEFT]{detection of exchangeable factors}
\acrodef{fg}[FG]{factor graph}
\acrodef{ljt}[LJT]{lifted junction tree}
\acrodef{lv}[logvar]{logical variable}
\acrodef{lve}[LVE]{lifted variable elimination}
\acrodef{mn}[MN]{Markov network}
\acrodef{pcrv}[PCRV]{parameterised CRV}
\acrodef{pf}[parfactor]{parametric factor}
\acrodef{pfg}[PFG]{parametric factor graph}
\acrodef{prv}[PRV]{parameterised randvar}
\acrodef{rv}[randvar]{random variable}
\acrodef{ve}[VE]{variable elimination}
\acrodef{wl}[WL]{Weisfeiler-Leman}

\crefname{algorithm}{Alg.}{Algs.}
\Crefname{algorithm}{Algorithm}{Algorithms}
\crefname{corollary}{Cor.}{Cors.}
\Crefname{corollary}{Corollary}{Corollaries}
\crefname{definition}{Def.}{Defs.}
\Crefname{definition}{Definition}{Definitions}
\crefname{example}{Ex.}{Exs.}
\Crefname{example}{Example}{Examples}
\crefname{proposition}{Prop.}{Props.}
\Crefname{proposition}{Proposition}{Propositions}
\crefname{section}{Sec.}{Secs.}
\Crefname{section}{Section}{Sections}
\crefname{theorem}{Thm.}{Thms.}
\Crefname{theorem}{Theorem}{Theorems}

\newcommand{\abs}[1]{\lvert #1 \rvert}
\newcommand{\cnt}{\ensuremath{\mathrm{count}}}
\newcommand{\range}[1]{\ensuremath{\mathrm{range}(#1)}}

\definecolor{myyellow}{RGB}{247,192,26}
\definecolor{myblue}{RGB}{37,122,164}
\definecolor{mygreen}{RGB}{78,155,133}
\definecolor{mypurple}{RGB}{86,51,94}

\definecolor{newblue}{RGB}{50,113,173}
\definecolor{newred}{RGB}{222,32,36}
\definecolor{newgreen}{RGB}{70,165,69}
\definecolor{newpurple}{RGB}{140,69,152}

\definecolor{cborange}{RGB}{230,159,0}
\definecolor{cbblue}{RGB}{30,136,229}
\definecolor{cbbluedark}{RGB}{46,37,133}
\definecolor{cbpurple}{RGB}{170,68,153}
\definecolor{cbgreen}{RGB}{0,77,64}
\definecolor{cbgreenlight}{RGB}{93,168,153}
\definecolor{cbbrown}{RGB}{126,41,84}

\pgfdeclarelayer{bg}
\pgfsetlayers{bg,main}

\tikzset{
	rv/.style={draw, ellipse},
	pf/.style={draw, rectangle, fill = gray!30},
	arc/.style = {->, >={[round,sep]Stealth}},
}

\newcommand\factor[6]{
	\node[pf, #1=#3 of #2, label={#4:{#5}}](#6) {};
}

\newcommand\nodecolorshift[5]{
	\node[circle, fill=#1, above right=0.1cm of #2, inner sep=0pt, minimum size=2mm, xshift=#4, yshift=#5](#3) {};
}

\newcommand\factorcolor[3]{
	\node[circle, fill=#1, above right=0cm and 0.05cm of #2, inner sep=0pt, minimum size=2mm](#3) {};
}
\newcommand\factorcolorshift[4]{
	\node[circle, fill=#1, above right=0cm and 0.05cm of #2, inner sep=0pt, minimum size=2mm, xshift=#4](#3) {};
}

\newcommand\pfs[8]{
	\node[pf, #1=#3 of #2, xshift=-1mm, yshift=1mm](#6) {};
	\node[pf, #1=#3 of #2, label={[label distance=1mm]#4:{#5}}](#7) {};
	\node[pf, #1=#3 of #2, xshift=1mm, yshift=-1mm](#8) {};
}

%% file: files/example_fg.tex
\begin{tikzpicture}
	\node[ellipse, draw] (A) {$R_1$};
	\node[ellipse, draw] (B) [below = 2.0em of A] {$R_2$};
	\factor{below right}{A}{1.00em and 2.0em}{270}{$\phi$}{f1}

	\node[right = 1.5em of f1] (tab_f1) {
		\begin{tabular}{c|c|c}
			$R_1$ & $R_2$ & $\phi(R_1,R_2)$ \\ \hline
			true  & true  & $\varphi_1$ \\
			true  & false & $\varphi_2$ \\
			false & true  & $\varphi_2$ \\
			false & false & $\varphi_3$ \\
		\end{tabular}
	};

	\draw (A) -- (f1);
	\draw (B) -- (f1);
\end{tikzpicture}

%% file: files/example_buckets.tex
\begin{tikzpicture}
	\node (tab_f1) {
		\begin{tabular}{c|c|c|c|c}
			$R_1$ & $R_2$ & $R_3$ & $\phi(R_1,R_2,R_3)$ & $b$ \\ \hline
			true  & true  & true  & $\varphi_1$         & {\color{cborange}$[3,0]$} \\
			true  & true  & false & $\varphi_2$         & {\color{cbbluedark}$[2,1]$} \\
			true  & false & true  & $\varphi_2$         & {\color{cbbluedark}$[2,1]$} \\
			true  & false & false & $\varphi_3$         & {\color{cbpurple}$[1,2]$} \\
			false & true  & true  & $\varphi_4$         & {\color{cbbluedark}$[2,1]$} \\
			false & true  & false & $\varphi_5$         & {\color{cbpurple}$[1,2]$} \\
			false & false & true  & $\varphi_5$         & {\color{cbpurple}$[1,2]$} \\
			false & false & false & $\varphi_6$         & {\color{cbgreen}$[0,3]$} \\
		\end{tabular}
	};

	\node[right = 0.2cm of tab_f1] (buckets) {
		\begin{tabular}{c|c}
			$b$                         & $\phi(b)$                                         \\ \hline
			{\color{cborange}$[3,0]$}   & $\langle \varphi_1 \rangle$                       \\
			{\color{cbbluedark}$[2,1]$} & $\langle \varphi_2, \varphi_2, \varphi_4 \rangle$ \\
			{\color{cbpurple}$[1,2]$}   & $\langle \varphi_3, \varphi_5, \varphi_5 \rangle$ \\
			{\color{cbgreen}$[0,3]$}    & $\langle \varphi_6 \rangle$                       \\
		\end{tabular}
	};
\end{tikzpicture}

%% file: files/example_buckets_crv.tex
\begin{tikzpicture}
	\node (tab_f1) {
		\begin{tabular}{c|c|c}
			$R_1$ & $\#_X[R(X)]$ & $\phi(R_1,\#_X[R(X)])$ \\ \hline
			true  & $[2,0]$      & $\varphi_1$         \\
			true  & $[1,1]$      & $\varphi_2$         \\
			true  & $[0,2]$      & $\varphi_3$         \\
			false & $[2,0]$      & $\varphi_4$         \\
			false & $[1,1]$      & $\varphi_5$         \\
			false & $[0,2]$      & $\varphi_6$         \\
		\end{tabular}
	};
\end{tikzpicture}

%% file: files/plot-avg.tex
\begin{tikzpicture}[x=1pt,y=1pt]
\definecolor{fillColor}{RGB}{255,255,255}
\path[use as bounding box,fill=fillColor,fill opacity=0.00] (0,0) rectangle (238.49,115.63);
\begin{scope}
\path[clip] (  0.00,  0.00) rectangle (238.49,115.63);
\definecolor{drawColor}{RGB}{255,255,255}
\definecolor{fillColor}{RGB}{255,255,255}

\path[draw=drawColor,line width= 0.6pt,line join=round,line cap=round,fill=fillColor] (  0.00,  0.00) rectangle (238.49,115.63);
\end{scope}
\begin{scope}
\path[clip] ( 44.91, 30.69) rectangle (232.99,110.13);
\definecolor{fillColor}{RGB}{255,255,255}

\path[fill=fillColor] ( 44.91, 30.69) rectangle (232.99,110.13);
\definecolor{drawColor}{RGB}{247,192,26}

\path[draw=drawColor,line width= 0.6pt,line join=round] ( 53.46, 34.57) --
	( 77.88, 41.99) --
	(102.31, 49.96) --
	(126.74, 58.26) --
	(151.16, 65.97) --
	(175.59, 74.33) --
	(200.02, 83.59) --
	(224.44, 93.81);
\definecolor{drawColor}{RGB}{37,122,164}

\path[draw=drawColor,line width= 0.6pt,dash pattern=on 2pt off 2pt ,line join=round] ( 53.46, 34.30) --
	( 77.88, 49.01) --
	(102.31, 63.21) --
	(126.74, 78.33) --
	(151.16, 92.52) --
	(175.59,106.52);
\definecolor{fillColor}{RGB}{247,192,26}

\path[fill=fillColor] (200.02, 83.59) circle (  1.96);

\path[fill=fillColor] ( 77.88, 41.99) circle (  1.96);

\path[fill=fillColor] (175.59, 74.33) circle (  1.96);

\path[fill=fillColor] (151.16, 65.97) circle (  1.96);

\path[fill=fillColor] ( 53.46, 34.57) circle (  1.96);

\path[fill=fillColor] (102.31, 49.96) circle (  1.96);

\path[fill=fillColor] (224.44, 93.81) circle (  1.96);

\path[fill=fillColor] (126.74, 58.26) circle (  1.96);
\definecolor{fillColor}{RGB}{37,122,164}

\path[fill=fillColor] ( 77.88, 52.07) --
	( 80.53, 47.49) --
	( 75.24, 47.49) --
	cycle;

\path[fill=fillColor] (175.59,109.57) --
	(178.23,105.00) --
	(172.95,105.00) --
	cycle;

\path[fill=fillColor] (151.16, 95.58) --
	(153.81, 91.00) --
	(148.52, 91.00) --
	cycle;

\path[fill=fillColor] ( 53.46, 37.35) --
	( 56.10, 32.77) --
	( 50.82, 32.77) --
	cycle;

\path[fill=fillColor] (102.31, 66.26) --
	(104.95, 61.69) --
	( 99.67, 61.69) --
	cycle;

\path[fill=fillColor] (126.74, 81.38) --
	(129.38, 76.81) --
	(124.09, 76.81) --
	cycle;
\end{scope}
\begin{scope}
\path[clip] (  0.00,  0.00) rectangle (238.49,115.63);
\definecolor{drawColor}{RGB}{0,0,0}

\path[draw=drawColor,line width= 0.6pt,line join=round] ( 44.91, 30.69) --
	( 44.91,110.13);

\path[draw=drawColor,line width= 0.6pt,line join=round] ( 46.33,107.67) --
	( 44.91,110.13) --
	( 43.49,107.67);
\end{scope}
\begin{scope}
\path[clip] (  0.00,  0.00) rectangle (238.49,115.63);
\definecolor{drawColor}{gray}{0.30}

\node[text=drawColor,anchor=base east,inner sep=0pt, outer sep=0pt, scale=  0.88] at ( 39.96, 52.55) {1};

\node[text=drawColor,anchor=base east,inner sep=0pt, outer sep=0pt, scale=  0.88] at ( 39.96, 74.80) {100};

\node[text=drawColor,anchor=base east,inner sep=0pt, outer sep=0pt, scale=  0.88] at ( 39.96, 97.04) {10000};
\end{scope}
\begin{scope}
\path[clip] (  0.00,  0.00) rectangle (238.49,115.63);
\definecolor{drawColor}{gray}{0.20}

\path[draw=drawColor,line width= 0.6pt,line join=round] ( 42.16, 55.58) --
	( 44.91, 55.58);

\path[draw=drawColor,line width= 0.6pt,line join=round] ( 42.16, 77.83) --
	( 44.91, 77.83);

\path[draw=drawColor,line width= 0.6pt,line join=round] ( 42.16,100.08) --
	( 44.91,100.08);
\end{scope}
\begin{scope}
\path[clip] (  0.00,  0.00) rectangle (238.49,115.63);
\definecolor{drawColor}{RGB}{0,0,0}

\path[draw=drawColor,line width= 0.6pt,line join=round] ( 44.91, 30.69) --
	(232.99, 30.69);

\path[draw=drawColor,line width= 0.6pt,line join=round] (230.53, 29.26) --
	(232.99, 30.69) --
	(230.53, 32.11);
\end{scope}
\begin{scope}
\path[clip] (  0.00,  0.00) rectangle (238.49,115.63);
\definecolor{drawColor}{gray}{0.20}

\path[draw=drawColor,line width= 0.6pt,line join=round] ( 77.88, 27.94) --
	( 77.88, 30.69);

\path[draw=drawColor,line width= 0.6pt,line join=round] (126.74, 27.94) --
	(126.74, 30.69);

\path[draw=drawColor,line width= 0.6pt,line join=round] (175.59, 27.94) --
	(175.59, 30.69);

\path[draw=drawColor,line width= 0.6pt,line join=round] (224.44, 27.94) --
	(224.44, 30.69);
\end{scope}
\begin{scope}
\path[clip] (  0.00,  0.00) rectangle (238.49,115.63);
\definecolor{drawColor}{gray}{0.30}

\node[text=drawColor,anchor=base,inner sep=0pt, outer sep=0pt, scale=  0.88] at ( 77.88, 19.68) {4};

\node[text=drawColor,anchor=base,inner sep=0pt, outer sep=0pt, scale=  0.88] at (126.74, 19.68) {8};

\node[text=drawColor,anchor=base,inner sep=0pt, outer sep=0pt, scale=  0.88] at (175.59, 19.68) {12};

\node[text=drawColor,anchor=base,inner sep=0pt, outer sep=0pt, scale=  0.88] at (224.44, 19.68) {16};
\end{scope}
\begin{scope}
\path[clip] (  0.00,  0.00) rectangle (238.49,115.63);
\definecolor{drawColor}{RGB}{0,0,0}

\node[text=drawColor,anchor=base,inner sep=0pt, outer sep=0pt, scale=  1.10] at (138.95,  7.64) {$n$};
\end{scope}
\begin{scope}
\path[clip] (  0.00,  0.00) rectangle (238.49,115.63);
\definecolor{drawColor}{RGB}{0,0,0}

\node[text=drawColor,rotate= 90.00,anchor=base,inner sep=0pt, outer sep=0pt, scale=  1.10] at ( 13.08, 70.41) {time (ms)};
\end{scope}
\begin{scope}
\path[clip] (  0.00,  0.00) rectangle (238.49,115.63);

\path[] ( 42.93, 74.29) rectangle (103.32,114.20);
\end{scope}
\begin{scope}
\path[clip] (  0.00,  0.00) rectangle (238.49,115.63);
\definecolor{drawColor}{RGB}{247,192,26}

\path[draw=drawColor,line width= 0.6pt,line join=round] ( 49.87,101.47) -- ( 61.43,101.47);
\end{scope}
\begin{scope}
\path[clip] (  0.00,  0.00) rectangle (238.49,115.63);
\definecolor{fillColor}{RGB}{247,192,26}

\path[fill=fillColor] ( 55.65,101.47) circle (  1.96);
\end{scope}
\begin{scope}
\path[clip] (  0.00,  0.00) rectangle (238.49,115.63);
\definecolor{drawColor}{RGB}{37,122,164}

\path[draw=drawColor,line width= 0.6pt,dash pattern=on 2pt off 2pt ,line join=round] ( 49.87, 87.02) -- ( 61.43, 87.02);
\end{scope}
\begin{scope}
\path[clip] (  0.00,  0.00) rectangle (238.49,115.63);
\definecolor{fillColor}{RGB}{37,122,164}

\path[fill=fillColor] ( 55.65, 90.07) --
	( 58.30, 85.49) --
	( 53.01, 85.49) --
	cycle;
\end{scope}
\begin{scope}
\path[clip] (  0.00,  0.00) rectangle (238.49,115.63);
\definecolor{drawColor}{RGB}{0,0,0}

\node[text=drawColor,anchor=base west,inner sep=0pt, outer sep=0pt, scale=  0.80] at ( 68.38, 98.71) {DECOR};
\end{scope}
\begin{scope}
\path[clip] (  0.00,  0.00) rectangle (238.49,115.63);
\definecolor{drawColor}{RGB}{0,0,0}

\node[text=drawColor,anchor=base west,inner sep=0pt, outer sep=0pt, scale=  0.80] at ( 68.38, 84.26) {Naive};
\end{scope}
\end{tikzpicture}

%% file: files/acp_example.tex
\begin{tikzpicture}[label distance=1mm]
	\node[circle, draw] (A) {$A$};
	\node[circle, draw] (B) [below = 0.5cm of A] {$B$};
	\node[circle, draw] (C) [below = 0.5cm of B] {$C$};
	\factor{below right}{A}{0.25cm and 0.5cm}{270}{$\phi_1$}{f1}
	\factor{below right}{B}{0.25cm and 0.5cm}{270}{$\phi_2$}{f2}

	\nodecolorshift{myyellow}{A}{Acol}{-2.1mm}{1mm}
	\nodecolorshift{myyellow}{B}{Bcol}{-2.1mm}{1mm}
	\nodecolorshift{myyellow}{C}{Ccol}{-2.1mm}{1mm}

	\factorcolor{myblue}{f1}{f1col}
	\factorcolor{myblue}{f2}{f2col}

	\draw (A) -- (f1);
	\draw (B) -- (f1);
	\draw (B) -- (f2);
	\draw (C) -- (f2);

	\node[circle, draw, right = 1.4cm of A] (A1) {$A$};
	\node[circle, draw, below = 0.5cm of A1] (B1) {$B$};
	\node[circle, draw, below = 0.5cm of B1] (C1) {$C$};
	\factor{below right}{A1}{0.25cm and 0.5cm}{270}{$\phi_1$}{f1_1}
	\factor{below right}{B1}{0.25cm and 0.5cm}{270}{$\phi_2$}{f2_1}

	\nodecolorshift{myyellow}{A1}{A1col}{-2.1mm}{1mm}
	\nodecolorshift{myyellow}{B1}{B1col}{-2.1mm}{1mm}
	\nodecolorshift{myyellow}{C1}{C1col}{-2.1mm}{1mm}

	\factorcolor{myyellow}{f1_1}{f1_1col1}
	\factorcolorshift{myyellow}{f1_1}{f1_1col2}{2.1mm}
	\factorcolorshift{myblue}{f1_1}{f1_1col3}{4.2mm}
	\factorcolor{myyellow}{f2_1}{f2_1col1}
	\factorcolorshift{myyellow}{f2_1}{f2_1col2}{2.1mm}
	\factorcolorshift{myblue}{f2_1}{f2_1col3}{4.2mm}

	\coordinate[right=0.1cm of A1, yshift=-0.1cm] (CA1);
	\coordinate[above=0.2cm of f1_1, yshift=-0.1cm] (Cf1_1);
	\coordinate[right=0.1cm of B1, yshift=0.12cm] (CB1);
	\coordinate[right=0.1cm of B1, yshift=-0.1cm] (CB1_1);
	\coordinate[below=0.2cm of f1_1, yshift=0.15cm] (Cf1_1b);
	\coordinate[above=0.2cm of f2_1, yshift=-0.1cm] (Cf2_1);
	\coordinate[right=0.1cm of C1, yshift=0.12cm] (CC1);
	\coordinate[below=0.2cm of f2_1, yshift=0.15cm] (Cf2_1b);

	\begin{pgfonlayer}{bg}
		\draw (A1) -- (f1_1);
		\draw [arc, gray] (CA1) -- (Cf1_1);
		\draw (B1) -- (f1_1);
		\draw [arc, gray] (CB1) -- (Cf1_1b);
		\draw (B1) -- (f2_1);
		\draw [arc, gray] (CB1_1) -- (Cf2_1);
		\draw (C1) -- (f2_1);
		\draw [arc, gray] (CC1) -- (Cf2_1b);
	\end{pgfonlayer}

	\node[circle, draw, right = 1.4cm of A1] (A2) {$A$};
	\node[circle, draw, below = 0.5cm of A2] (B2) {$B$};
	\node[circle, draw, below = 0.5cm of B2] (C2) {$C$};
	\factor{below right}{A2}{0.25cm and 0.5cm}{270}{$\phi_1$}{f1_2}
	\factor{below right}{B2}{0.25cm and 0.5cm}{270}{$\phi_2$}{f2_2}

	\nodecolorshift{myyellow}{A2}{A2col}{-2.1mm}{1mm}
	\nodecolorshift{myyellow}{B2}{B2col}{-2.1mm}{1mm}
	\nodecolorshift{myyellow}{C2}{C2col}{-2.1mm}{1mm}

	\factorcolor{myblue}{f1_2}{f1_2col1}
	\factorcolor{myblue}{f2_2}{f2_2col1}

	\draw (A2) -- (f1_2);
	\draw (B2) -- (f1_2);
	\draw (B2) -- (f2_2);
	\draw (C2) -- (f2_2);

	\node[circle, draw, right = 1.4cm of A2] (A3) {$A$};
	\node[circle, draw, below = 0.5cm of A3] (B3) {$B$};
	\node[circle, draw, below = 0.5cm of B3] (C3) {$C$};
	\factor{below right}{A3}{0.25cm and 0.5cm}{270}{$\phi_1$}{f1_3}
	\factor{below right}{B3}{0.25cm and 0.5cm}{270}{$\phi_2$}{f2_3}

	\nodecolorshift{myblue}{A3}{A3col1}{-4.2mm}{1mm}
	\nodecolorshift{myyellow}{A3}{A3col2}{-2.1mm}{1mm}
	\nodecolorshift{myblue}{B3}{B3col1}{-6.3mm}{1mm}
	\nodecolorshift{myblue}{B3}{B3col2}{-4.2mm}{1mm}
	\nodecolorshift{myyellow}{B3}{B3col3}{-2.1mm}{1mm}
	\nodecolorshift{myblue}{C3}{C3col1}{-4.2mm}{1mm}
	\nodecolorshift{myyellow}{C3}{C3col2}{-2.1mm}{1mm}

	\factorcolor{myblue}{f1_3}{f1_3col1}
	\factorcolor{myblue}{f2_3}{f2_3col1}

	\coordinate[right=0.1cm of A3, yshift=-0.1cm] (CA3);
	\coordinate[above=0.2cm of f1_3, yshift=-0.1cm] (Cf1_3);
	\coordinate[right=0.1cm of B3, yshift=0.12cm] (CB3);
	\coordinate[right=0.1cm of B3, yshift=-0.1cm] (CB1_3);
	\coordinate[below=0.2cm of f1_3, yshift=0.15cm] (Cf1_3b);
	\coordinate[above=0.2cm of f2_3, yshift=-0.1cm] (Cf2_3);
	\coordinate[right=0.1cm of C3, yshift=0.12cm] (CC3);
	\coordinate[below=0.2cm of f2_3, yshift=0.15cm] (Cf2_3b);

	\begin{pgfonlayer}{bg}
		\draw (A3) -- (f1_3);
		\draw [arc, gray] (Cf1_3) -- (CA3);
		\draw (B3) -- (f1_3);
		\draw [arc, gray] (Cf1_3b) -- (CB3);
		\draw (B3) -- (f2_3);
		\draw [arc, gray] (Cf2_3) -- (CB1_3);
		\draw (C3) -- (f2_3);
		\draw [arc, gray] (Cf2_3b) -- (CC3);
	\end{pgfonlayer}

	\node[circle, draw, right = 1.4cm of A3] (A4) {$A$};
	\node[circle, draw, below = 0.5cm of A4] (B4) {$B$};
	\node[circle, draw, below = 0.5cm of B4] (C4) {$C$};
	\factor{below right}{A4}{0.25cm and 0.5cm}{270}{$\phi_1$}{f1_4}
	\factor{below right}{B4}{0.25cm and 0.5cm}{270}{$\phi_2$}{f2_4}

	\nodecolorshift{myyellow}{A4}{A4col}{-2.1mm}{1mm}
	\nodecolorshift{mygreen}{B4}{B4col}{-2.1mm}{1mm}
	\nodecolorshift{myyellow}{C4}{C4col}{-2.1mm}{1mm}

	\factorcolor{myblue}{f1_4}{f1_4col1}
	\factorcolor{myblue}{f2_4}{f2_4col1}

	\draw (A4) -- (f1_4);
	\draw (B4) -- (f1_4);
	\draw (B4) -- (f2_4);
	\draw (C4) -- (f2_4);

	\pfs{right}{B4}{2.9cm}{270}{$\phi'_1$}{f12a}{f12}{f12b}

	\node[ellipse, inner sep = 1.2pt, draw, above left = 0.25cm and 0.5cm of f12] (AC) {$R(X)$};
	\node[circle, draw] (B) [below left = 0.25cm and 0.7cm of f12] {$B$};

	\begin{pgfonlayer}{bg}
		\draw (AC) -- (f12);
		\draw (B) -- (f12);
	\end{pgfonlayer}

	\node[below = 0.5cm of C2, xshift=1.5cm] (tab_f2) {
		\begin{tabular}{c|c|c}
			$C$   & $B$   & $\phi_2(C,B)$ \\ \hline
			true  & true  & $\varphi_1$ \\
			true  & false & $\varphi_2$ \\
			false & true  & $\varphi_3$ \\
			false & false & $\varphi_4$ \\
		\end{tabular}
	};

	\node[left = 0.4cm of tab_f2] (tab_f1) {
		\begin{tabular}{c|c|c}
			$A$   & $B$   & $\phi_1(A,B)$ \\ \hline
			true  & true  & $\varphi_1$ \\
			true  & false & $\varphi_2$ \\
			false & true  & $\varphi_3$ \\
			false & false & $\varphi_4$ \\
		\end{tabular}
	};

	\node[right = 0.4cm of tab_f2] (tab_f12) {
		\begin{tabular}{c|c|c}
			$R(X)$   & $B$   & $\phi'_1(R(X),B)$ \\ \hline
			true  & true  & $\varphi_1$ \\
			true  & false & $\varphi_2$ \\
			false & true  & $\varphi_3$ \\
			false & false & $\varphi_4$ \\
		\end{tabular}
	};
\end{tikzpicture}

%% file: files/plot-k=0.tex
\begin{tikzpicture}[x=1pt,y=1pt]
\definecolor{fillColor}{RGB}{255,255,255}
\path[use as bounding box,fill=fillColor,fill opacity=0.00] (0,0) rectangle (209.58,115.63);
\begin{scope}
\path[clip] (  0.00,  0.00) rectangle (209.58,115.63);
\definecolor{drawColor}{RGB}{255,255,255}
\definecolor{fillColor}{RGB}{255,255,255}

\path[draw=drawColor,line width= 0.6pt,line join=round,line cap=round,fill=fillColor] (  0.00,  0.00) rectangle (209.58,115.63);
\end{scope}
\begin{scope}
\path[clip] ( 44.91, 30.69) rectangle (204.08,110.13);
\definecolor{fillColor}{RGB}{255,255,255}

\path[fill=fillColor] ( 44.91, 30.69) rectangle (204.08,110.13);
\definecolor{drawColor}{RGB}{247,192,26}

\path[draw=drawColor,line width= 0.6pt,line join=round] ( 52.14, 34.30) --
	( 72.82, 39.62) --
	( 93.49, 46.24) --
	(114.16, 53.74) --
	(134.83, 60.10) --
	(155.50, 69.94) --
	(176.18, 74.07) --
	(196.85, 81.14);
\definecolor{drawColor}{RGB}{37,122,164}

\path[draw=drawColor,line width= 0.6pt,dash pattern=on 2pt off 2pt ,line join=round] ( 52.14, 34.93) --
	( 72.82, 51.76) --
	( 93.49, 65.90) --
	(114.16, 79.61) --
	(134.83, 93.08) --
	(155.50,106.52);
\definecolor{fillColor}{RGB}{247,192,26}

\path[fill=fillColor] (176.18, 74.07) circle (  1.96);

\path[fill=fillColor] ( 72.82, 39.62) circle (  1.96);

\path[fill=fillColor] (155.50, 69.94) circle (  1.96);

\path[fill=fillColor] (134.83, 60.10) circle (  1.96);

\path[fill=fillColor] ( 52.14, 34.30) circle (  1.96);

\path[fill=fillColor] ( 93.49, 46.24) circle (  1.96);

\path[fill=fillColor] (196.85, 81.14) circle (  1.96);

\path[fill=fillColor] (114.16, 53.74) circle (  1.96);
\definecolor{fillColor}{RGB}{37,122,164}

\path[fill=fillColor] ( 72.82, 54.82) --
	( 75.46, 50.24) --
	( 70.17, 50.24) --
	cycle;

\path[fill=fillColor] (155.50,109.57) --
	(158.15,105.00) --
	(152.86,105.00) --
	cycle;

\path[fill=fillColor] (134.83, 96.13) --
	(137.47, 91.56) --
	(132.19, 91.56) --
	cycle;

\path[fill=fillColor] ( 52.14, 37.98) --
	( 54.79, 33.41) --
	( 49.50, 33.41) --
	cycle;

\path[fill=fillColor] ( 93.49, 68.96) --
	( 96.13, 64.38) --
	( 90.85, 64.38) --
	cycle;

\path[fill=fillColor] (114.16, 82.67) --
	(116.80, 78.09) --
	(111.52, 78.09) --
	cycle;
\end{scope}
\begin{scope}
\path[clip] (  0.00,  0.00) rectangle (209.58,115.63);
\definecolor{drawColor}{RGB}{0,0,0}

\path[draw=drawColor,line width= 0.6pt,line join=round] ( 44.91, 30.69) --
	( 44.91,110.13);

\path[draw=drawColor,line width= 0.6pt,line join=round] ( 46.33,107.67) --
	( 44.91,110.13) --
	( 43.49,107.67);
\end{scope}
\begin{scope}
\path[clip] (  0.00,  0.00) rectangle (209.58,115.63);
\definecolor{drawColor}{gray}{0.30}

\node[text=drawColor,anchor=base east,inner sep=0pt, outer sep=0pt, scale=  0.88] at ( 39.96, 52.36) {1};

\node[text=drawColor,anchor=base east,inner sep=0pt, outer sep=0pt, scale=  0.88] at ( 39.96, 73.61) {100};

\node[text=drawColor,anchor=base east,inner sep=0pt, outer sep=0pt, scale=  0.88] at ( 39.96, 94.87) {10000};
\end{scope}
\begin{scope}
\path[clip] (  0.00,  0.00) rectangle (209.58,115.63);
\definecolor{drawColor}{gray}{0.20}

\path[draw=drawColor,line width= 0.6pt,line join=round] ( 42.16, 55.39) --
	( 44.91, 55.39);

\path[draw=drawColor,line width= 0.6pt,line join=round] ( 42.16, 76.64) --
	( 44.91, 76.64);

\path[draw=drawColor,line width= 0.6pt,line join=round] ( 42.16, 97.90) --
	( 44.91, 97.90);
\end{scope}
\begin{scope}
\path[clip] (  0.00,  0.00) rectangle (209.58,115.63);
\definecolor{drawColor}{RGB}{0,0,0}

\path[draw=drawColor,line width= 0.6pt,line join=round] ( 44.91, 30.69) --
	(204.08, 30.69);

\path[draw=drawColor,line width= 0.6pt,line join=round] (201.62, 29.26) --
	(204.08, 30.69) --
	(201.62, 32.11);
\end{scope}
\begin{scope}
\path[clip] (  0.00,  0.00) rectangle (209.58,115.63);
\definecolor{drawColor}{gray}{0.20}

\path[draw=drawColor,line width= 0.6pt,line join=round] ( 72.82, 27.94) --
	( 72.82, 30.69);

\path[draw=drawColor,line width= 0.6pt,line join=round] (114.16, 27.94) --
	(114.16, 30.69);

\path[draw=drawColor,line width= 0.6pt,line join=round] (155.50, 27.94) --
	(155.50, 30.69);

\path[draw=drawColor,line width= 0.6pt,line join=round] (196.85, 27.94) --
	(196.85, 30.69);
\end{scope}
\begin{scope}
\path[clip] (  0.00,  0.00) rectangle (209.58,115.63);
\definecolor{drawColor}{gray}{0.30}

\node[text=drawColor,anchor=base,inner sep=0pt, outer sep=0pt, scale=  0.88] at ( 72.82, 19.68) {4};

\node[text=drawColor,anchor=base,inner sep=0pt, outer sep=0pt, scale=  0.88] at (114.16, 19.68) {8};

\node[text=drawColor,anchor=base,inner sep=0pt, outer sep=0pt, scale=  0.88] at (155.50, 19.68) {12};

\node[text=drawColor,anchor=base,inner sep=0pt, outer sep=0pt, scale=  0.88] at (196.85, 19.68) {16};
\end{scope}
\begin{scope}
\path[clip] (  0.00,  0.00) rectangle (209.58,115.63);
\definecolor{drawColor}{RGB}{0,0,0}

\node[text=drawColor,anchor=base,inner sep=0pt, outer sep=0pt, scale=  1.10] at (124.50,  7.64) {$n$};
\end{scope}
\begin{scope}
\path[clip] (  0.00,  0.00) rectangle (209.58,115.63);
\definecolor{drawColor}{RGB}{0,0,0}

\node[text=drawColor,rotate= 90.00,anchor=base,inner sep=0pt, outer sep=0pt, scale=  1.10] at ( 13.08, 70.41) {time (ms)};
\end{scope}
\begin{scope}
\path[clip] (  0.00,  0.00) rectangle (209.58,115.63);

\path[] ( 41.77, 74.29) rectangle (102.16,114.20);
\end{scope}
\begin{scope}
\path[clip] (  0.00,  0.00) rectangle (209.58,115.63);
\definecolor{drawColor}{RGB}{247,192,26}

\path[draw=drawColor,line width= 0.6pt,line join=round] ( 48.72,101.47) -- ( 60.28,101.47);
\end{scope}
\begin{scope}
\path[clip] (  0.00,  0.00) rectangle (209.58,115.63);
\definecolor{fillColor}{RGB}{247,192,26}

\path[fill=fillColor] ( 54.50,101.47) circle (  1.96);
\end{scope}
\begin{scope}
\path[clip] (  0.00,  0.00) rectangle (209.58,115.63);
\definecolor{drawColor}{RGB}{37,122,164}

\path[draw=drawColor,line width= 0.6pt,dash pattern=on 2pt off 2pt ,line join=round] ( 48.72, 87.02) -- ( 60.28, 87.02);
\end{scope}
\begin{scope}
\path[clip] (  0.00,  0.00) rectangle (209.58,115.63);
\definecolor{fillColor}{RGB}{37,122,164}

\path[fill=fillColor] ( 54.50, 90.07) --
	( 57.14, 85.49) --
	( 51.86, 85.49) --
	cycle;
\end{scope}
\begin{scope}
\path[clip] (  0.00,  0.00) rectangle (209.58,115.63);
\definecolor{drawColor}{RGB}{0,0,0}

\node[text=drawColor,anchor=base west,inner sep=0pt, outer sep=0pt, scale=  0.80] at ( 67.23, 98.71) {DECOR};
\end{scope}
\begin{scope}
\path[clip] (  0.00,  0.00) rectangle (209.58,115.63);
\definecolor{drawColor}{RGB}{0,0,0}

\node[text=drawColor,anchor=base west,inner sep=0pt, outer sep=0pt, scale=  0.80] at ( 67.23, 84.26) {Naive};
\end{scope}
\end{tikzpicture}

%% file: files/plot-k=2.tex
\begin{tikzpicture}[x=1pt,y=1pt]
\definecolor{fillColor}{RGB}{255,255,255}
\path[use as bounding box,fill=fillColor,fill opacity=0.00] (0,0) rectangle (209.58,115.63);
\begin{scope}
\path[clip] (  0.00,  0.00) rectangle (209.58,115.63);
\definecolor{drawColor}{RGB}{255,255,255}
\definecolor{fillColor}{RGB}{255,255,255}

\path[draw=drawColor,line width= 0.6pt,line join=round,line cap=round,fill=fillColor] (  0.00,  0.00) rectangle (209.58,115.63);
\end{scope}
\begin{scope}
\path[clip] ( 44.91, 30.69) rectangle (204.08,110.13);
\definecolor{fillColor}{RGB}{255,255,255}

\path[fill=fillColor] ( 44.91, 30.69) rectangle (204.08,110.13);
\definecolor{drawColor}{RGB}{247,192,26}

\path[draw=drawColor,line width= 0.6pt,line join=round] ( 52.14, 35.28) --
	( 72.82, 41.67) --
	( 93.49, 48.96) --
	(114.16, 56.82) --
	(134.83, 64.82) --
	(155.50, 74.11) --
	(176.18, 84.97) --
	(196.85, 96.88);
\definecolor{drawColor}{RGB}{37,122,164}

\path[draw=drawColor,line width= 0.6pt,dash pattern=on 2pt off 2pt ,line join=round] ( 52.14, 34.30) --
	( 72.82, 50.18) --
	( 93.49, 64.91) --
	(114.16, 79.17) --
	(134.83, 92.95) --
	(155.50,106.52);
\definecolor{fillColor}{RGB}{247,192,26}

\path[fill=fillColor] (176.18, 84.97) circle (  1.96);

\path[fill=fillColor] ( 72.82, 41.67) circle (  1.96);

\path[fill=fillColor] (155.50, 74.11) circle (  1.96);

\path[fill=fillColor] (134.83, 64.82) circle (  1.96);

\path[fill=fillColor] ( 52.14, 35.28) circle (  1.96);

\path[fill=fillColor] ( 93.49, 48.96) circle (  1.96);

\path[fill=fillColor] (196.85, 96.88) circle (  1.96);

\path[fill=fillColor] (114.16, 56.82) circle (  1.96);
\definecolor{fillColor}{RGB}{37,122,164}

\path[fill=fillColor] ( 72.82, 53.24) --
	( 75.46, 48.66) --
	( 70.17, 48.66) --
	cycle;

\path[fill=fillColor] (155.50,109.57) --
	(158.15,105.00) --
	(152.86,105.00) --
	cycle;

\path[fill=fillColor] (134.83, 96.00) --
	(137.47, 91.42) --
	(132.19, 91.42) --
	cycle;

\path[fill=fillColor] ( 52.14, 37.35) --
	( 54.79, 32.77) --
	( 49.50, 32.77) --
	cycle;

\path[fill=fillColor] ( 93.49, 67.97) --
	( 96.13, 63.39) --
	( 90.85, 63.39) --
	cycle;

\path[fill=fillColor] (114.16, 82.22) --
	(116.80, 77.64) --
	(111.52, 77.64) --
	cycle;
\end{scope}
\begin{scope}
\path[clip] (  0.00,  0.00) rectangle (209.58,115.63);
\definecolor{drawColor}{RGB}{0,0,0}

\path[draw=drawColor,line width= 0.6pt,line join=round] ( 44.91, 30.69) --
	( 44.91,110.13);

\path[draw=drawColor,line width= 0.6pt,line join=round] ( 46.33,107.67) --
	( 44.91,110.13) --
	( 43.49,107.67);
\end{scope}
\begin{scope}
\path[clip] (  0.00,  0.00) rectangle (209.58,115.63);
\definecolor{drawColor}{gray}{0.30}

\node[text=drawColor,anchor=base east,inner sep=0pt, outer sep=0pt, scale=  0.88] at ( 39.96, 51.72) {1};

\node[text=drawColor,anchor=base east,inner sep=0pt, outer sep=0pt, scale=  0.88] at ( 39.96, 73.23) {100};

\node[text=drawColor,anchor=base east,inner sep=0pt, outer sep=0pt, scale=  0.88] at ( 39.96, 94.73) {10000};
\end{scope}
\begin{scope}
\path[clip] (  0.00,  0.00) rectangle (209.58,115.63);
\definecolor{drawColor}{gray}{0.20}

\path[draw=drawColor,line width= 0.6pt,line join=round] ( 42.16, 54.75) --
	( 44.91, 54.75);

\path[draw=drawColor,line width= 0.6pt,line join=round] ( 42.16, 76.26) --
	( 44.91, 76.26);

\path[draw=drawColor,line width= 0.6pt,line join=round] ( 42.16, 97.76) --
	( 44.91, 97.76);
\end{scope}
\begin{scope}
\path[clip] (  0.00,  0.00) rectangle (209.58,115.63);
\definecolor{drawColor}{RGB}{0,0,0}

\path[draw=drawColor,line width= 0.6pt,line join=round] ( 44.91, 30.69) --
	(204.08, 30.69);

\path[draw=drawColor,line width= 0.6pt,line join=round] (201.62, 29.26) --
	(204.08, 30.69) --
	(201.62, 32.11);
\end{scope}
\begin{scope}
\path[clip] (  0.00,  0.00) rectangle (209.58,115.63);
\definecolor{drawColor}{gray}{0.20}

\path[draw=drawColor,line width= 0.6pt,line join=round] ( 72.82, 27.94) --
	( 72.82, 30.69);

\path[draw=drawColor,line width= 0.6pt,line join=round] (114.16, 27.94) --
	(114.16, 30.69);

\path[draw=drawColor,line width= 0.6pt,line join=round] (155.50, 27.94) --
	(155.50, 30.69);

\path[draw=drawColor,line width= 0.6pt,line join=round] (196.85, 27.94) --
	(196.85, 30.69);
\end{scope}
\begin{scope}
\path[clip] (  0.00,  0.00) rectangle (209.58,115.63);
\definecolor{drawColor}{gray}{0.30}

\node[text=drawColor,anchor=base,inner sep=0pt, outer sep=0pt, scale=  0.88] at ( 72.82, 19.68) {4};

\node[text=drawColor,anchor=base,inner sep=0pt, outer sep=0pt, scale=  0.88] at (114.16, 19.68) {8};

\node[text=drawColor,anchor=base,inner sep=0pt, outer sep=0pt, scale=  0.88] at (155.50, 19.68) {12};

\node[text=drawColor,anchor=base,inner sep=0pt, outer sep=0pt, scale=  0.88] at (196.85, 19.68) {16};
\end{scope}
\begin{scope}
\path[clip] (  0.00,  0.00) rectangle (209.58,115.63);
\definecolor{drawColor}{RGB}{0,0,0}

\node[text=drawColor,anchor=base,inner sep=0pt, outer sep=0pt, scale=  1.10] at (124.50,  7.64) {$n$};
\end{scope}
\begin{scope}
\path[clip] (  0.00,  0.00) rectangle (209.58,115.63);
\definecolor{drawColor}{RGB}{0,0,0}

\node[text=drawColor,rotate= 90.00,anchor=base,inner sep=0pt, outer sep=0pt, scale=  1.10] at ( 13.08, 70.41) {time (ms)};
\end{scope}
\begin{scope}
\path[clip] (  0.00,  0.00) rectangle (209.58,115.63);

\path[] ( 41.77, 74.29) rectangle (102.16,114.20);
\end{scope}
\begin{scope}
\path[clip] (  0.00,  0.00) rectangle (209.58,115.63);
\definecolor{drawColor}{RGB}{247,192,26}

\path[draw=drawColor,line width= 0.6pt,line join=round] ( 48.72,101.47) -- ( 60.28,101.47);
\end{scope}
\begin{scope}
\path[clip] (  0.00,  0.00) rectangle (209.58,115.63);
\definecolor{fillColor}{RGB}{247,192,26}

\path[fill=fillColor] ( 54.50,101.47) circle (  1.96);
\end{scope}
\begin{scope}
\path[clip] (  0.00,  0.00) rectangle (209.58,115.63);
\definecolor{drawColor}{RGB}{37,122,164}

\path[draw=drawColor,line width= 0.6pt,dash pattern=on 2pt off 2pt ,line join=round] ( 48.72, 87.02) -- ( 60.28, 87.02);
\end{scope}
\begin{scope}
\path[clip] (  0.00,  0.00) rectangle (209.58,115.63);
\definecolor{fillColor}{RGB}{37,122,164}

\path[fill=fillColor] ( 54.50, 90.07) --
	( 57.14, 85.49) --
	( 51.86, 85.49) --
	cycle;
\end{scope}
\begin{scope}
\path[clip] (  0.00,  0.00) rectangle (209.58,115.63);
\definecolor{drawColor}{RGB}{0,0,0}

\node[text=drawColor,anchor=base west,inner sep=0pt, outer sep=0pt, scale=  0.80] at ( 67.23, 98.71) {DECOR};
\end{scope}
\begin{scope}
\path[clip] (  0.00,  0.00) rectangle (209.58,115.63);
\definecolor{drawColor}{RGB}{0,0,0}

\node[text=drawColor,anchor=base west,inner sep=0pt, outer sep=0pt, scale=  0.80] at ( 67.23, 84.26) {Naive};
\end{scope}
\end{tikzpicture}

%% file: files/plot-k=log2n.tex
\begin{tikzpicture}[x=1pt,y=1pt]
\definecolor{fillColor}{RGB}{255,255,255}
\path[use as bounding box,fill=fillColor,fill opacity=0.00] (0,0) rectangle (209.58,115.63);
\begin{scope}
\path[clip] (  0.00,  0.00) rectangle (209.58,115.63);
\definecolor{drawColor}{RGB}{255,255,255}
\definecolor{fillColor}{RGB}{255,255,255}

\path[draw=drawColor,line width= 0.6pt,line join=round,line cap=round,fill=fillColor] (  0.00,  0.00) rectangle (209.58,115.63);
\end{scope}
\begin{scope}
\path[clip] ( 44.91, 30.69) rectangle (204.08,110.13);
\definecolor{fillColor}{RGB}{255,255,255}

\path[fill=fillColor] ( 44.91, 30.69) rectangle (204.08,110.13);
\definecolor{drawColor}{RGB}{247,192,26}

\path[draw=drawColor,line width= 0.6pt,line join=round] ( 52.14, 34.30) --
	( 76.26, 42.42) --
	(100.38, 52.18) --
	(124.50, 60.70) --
	(148.61, 70.21) --
	(172.73, 81.38) --
	(196.85, 92.16);
\definecolor{drawColor}{RGB}{37,122,164}

\path[draw=drawColor,line width= 0.6pt,dash pattern=on 2pt off 2pt ,line join=round] ( 52.14, 43.80) --
	( 76.26, 60.24) --
	(100.38, 75.49) --
	(124.50, 91.15) --
	(148.61,106.52);
\definecolor{fillColor}{RGB}{247,192,26}

\path[fill=fillColor] (172.73, 81.38) circle (  1.96);

\path[fill=fillColor] ( 52.14, 34.30) circle (  1.96);

\path[fill=fillColor] (148.61, 70.21) circle (  1.96);

\path[fill=fillColor] (124.50, 60.70) circle (  1.96);

\path[fill=fillColor] ( 76.26, 42.42) circle (  1.96);

\path[fill=fillColor] (196.85, 92.16) circle (  1.96);

\path[fill=fillColor] (100.38, 52.18) circle (  1.96);
\definecolor{fillColor}{RGB}{37,122,164}

\path[fill=fillColor] ( 52.14, 46.85) --
	( 54.79, 42.27) --
	( 49.50, 42.27) --
	cycle;

\path[fill=fillColor] (148.61,109.57) --
	(151.26,105.00) --
	(145.97,105.00) --
	cycle;

\path[fill=fillColor] (124.50, 94.20) --
	(127.14, 89.62) --
	(121.85, 89.62) --
	cycle;

\path[fill=fillColor] ( 76.26, 63.29) --
	( 78.90, 58.71) --
	( 73.62, 58.71) --
	cycle;

\path[fill=fillColor] (100.38, 78.54) --
	(103.02, 73.97) --
	( 97.74, 73.97) --
	cycle;
\end{scope}
\begin{scope}
\path[clip] (  0.00,  0.00) rectangle (209.58,115.63);
\definecolor{drawColor}{RGB}{0,0,0}

\path[draw=drawColor,line width= 0.6pt,line join=round] ( 44.91, 30.69) --
	( 44.91,110.13);

\path[draw=drawColor,line width= 0.6pt,line join=round] ( 46.33,107.67) --
	( 44.91,110.13) --
	( 43.49,107.67);
\end{scope}
\begin{scope}
\path[clip] (  0.00,  0.00) rectangle (209.58,115.63);
\definecolor{drawColor}{gray}{0.30}

\node[text=drawColor,anchor=base east,inner sep=0pt, outer sep=0pt, scale=  0.88] at ( 39.96, 45.87) {1};

\node[text=drawColor,anchor=base east,inner sep=0pt, outer sep=0pt, scale=  0.88] at ( 39.96, 69.87) {100};

\node[text=drawColor,anchor=base east,inner sep=0pt, outer sep=0pt, scale=  0.88] at ( 39.96, 93.86) {10000};
\end{scope}
\begin{scope}
\path[clip] (  0.00,  0.00) rectangle (209.58,115.63);
\definecolor{drawColor}{gray}{0.20}

\path[draw=drawColor,line width= 0.6pt,line join=round] ( 42.16, 48.90) --
	( 44.91, 48.90);

\path[draw=drawColor,line width= 0.6pt,line join=round] ( 42.16, 72.90) --
	( 44.91, 72.90);

\path[draw=drawColor,line width= 0.6pt,line join=round] ( 42.16, 96.90) --
	( 44.91, 96.90);
\end{scope}
\begin{scope}
\path[clip] (  0.00,  0.00) rectangle (209.58,115.63);
\definecolor{drawColor}{RGB}{0,0,0}

\path[draw=drawColor,line width= 0.6pt,line join=round] ( 44.91, 30.69) --
	(204.08, 30.69);

\path[draw=drawColor,line width= 0.6pt,line join=round] (201.62, 29.26) --
	(204.08, 30.69) --
	(201.62, 32.11);
\end{scope}
\begin{scope}
\path[clip] (  0.00,  0.00) rectangle (209.58,115.63);
\definecolor{drawColor}{gray}{0.20}

\path[draw=drawColor,line width= 0.6pt,line join=round] ( 52.14, 27.94) --
	( 52.14, 30.69);

\path[draw=drawColor,line width= 0.6pt,line join=round] (100.38, 27.94) --
	(100.38, 30.69);

\path[draw=drawColor,line width= 0.6pt,line join=round] (148.61, 27.94) --
	(148.61, 30.69);

\path[draw=drawColor,line width= 0.6pt,line join=round] (196.85, 27.94) --
	(196.85, 30.69);
\end{scope}
\begin{scope}
\path[clip] (  0.00,  0.00) rectangle (209.58,115.63);
\definecolor{drawColor}{gray}{0.30}

\node[text=drawColor,anchor=base,inner sep=0pt, outer sep=0pt, scale=  0.88] at ( 52.14, 19.68) {4};

\node[text=drawColor,anchor=base,inner sep=0pt, outer sep=0pt, scale=  0.88] at (100.38, 19.68) {8};

\node[text=drawColor,anchor=base,inner sep=0pt, outer sep=0pt, scale=  0.88] at (148.61, 19.68) {12};

\node[text=drawColor,anchor=base,inner sep=0pt, outer sep=0pt, scale=  0.88] at (196.85, 19.68) {16};
\end{scope}
\begin{scope}
\path[clip] (  0.00,  0.00) rectangle (209.58,115.63);
\definecolor{drawColor}{RGB}{0,0,0}

\node[text=drawColor,anchor=base,inner sep=0pt, outer sep=0pt, scale=  1.10] at (124.50,  7.64) {$n$};
\end{scope}
\begin{scope}
\path[clip] (  0.00,  0.00) rectangle (209.58,115.63);
\definecolor{drawColor}{RGB}{0,0,0}

\node[text=drawColor,rotate= 90.00,anchor=base,inner sep=0pt, outer sep=0pt, scale=  1.10] at ( 13.08, 70.41) {time (ms)};
\end{scope}
\begin{scope}
\path[clip] (  0.00,  0.00) rectangle (209.58,115.63);

\path[] ( 41.77, 74.29) rectangle (102.16,114.20);
\end{scope}
\begin{scope}
\path[clip] (  0.00,  0.00) rectangle (209.58,115.63);
\definecolor{drawColor}{RGB}{247,192,26}

\path[draw=drawColor,line width= 0.6pt,line join=round] ( 48.72,101.47) -- ( 60.28,101.47);
\end{scope}
\begin{scope}
\path[clip] (  0.00,  0.00) rectangle (209.58,115.63);
\definecolor{fillColor}{RGB}{247,192,26}

\path[fill=fillColor] ( 54.50,101.47) circle (  1.96);
\end{scope}
\begin{scope}
\path[clip] (  0.00,  0.00) rectangle (209.58,115.63);
\definecolor{drawColor}{RGB}{37,122,164}

\path[draw=drawColor,line width= 0.6pt,dash pattern=on 2pt off 2pt ,line join=round] ( 48.72, 87.02) -- ( 60.28, 87.02);
\end{scope}
\begin{scope}
\path[clip] (  0.00,  0.00) rectangle (209.58,115.63);
\definecolor{fillColor}{RGB}{37,122,164}

\path[fill=fillColor] ( 54.50, 90.07) --
	( 57.14, 85.49) --
	( 51.86, 85.49) --
	cycle;
\end{scope}
\begin{scope}
\path[clip] (  0.00,  0.00) rectangle (209.58,115.63);
\definecolor{drawColor}{RGB}{0,0,0}

\node[text=drawColor,anchor=base west,inner sep=0pt, outer sep=0pt, scale=  0.80] at ( 67.23, 98.71) {DECOR};
\end{scope}
\begin{scope}
\path[clip] (  0.00,  0.00) rectangle (209.58,115.63);
\definecolor{drawColor}{RGB}{0,0,0}

\node[text=drawColor,anchor=base west,inner sep=0pt, outer sep=0pt, scale=  0.80] at ( 67.23, 84.26) {Naive};
\end{scope}
\end{tikzpicture}

%% file: files/plot-k=ndiv2.tex
\begin{tikzpicture}[x=1pt,y=1pt]
\definecolor{fillColor}{RGB}{255,255,255}
\path[use as bounding box,fill=fillColor,fill opacity=0.00] (0,0) rectangle (209.58,115.63);
\begin{scope}
\path[clip] (  0.00,  0.00) rectangle (209.58,115.63);
\definecolor{drawColor}{RGB}{255,255,255}
\definecolor{fillColor}{RGB}{255,255,255}

\path[draw=drawColor,line width= 0.6pt,line join=round,line cap=round,fill=fillColor] (  0.00,  0.00) rectangle (209.58,115.63);
\end{scope}
\begin{scope}
\path[clip] ( 44.91, 30.69) rectangle (204.08,110.13);
\definecolor{fillColor}{RGB}{255,255,255}

\path[fill=fillColor] ( 44.91, 30.69) rectangle (204.08,110.13);
\definecolor{drawColor}{RGB}{247,192,26}

\path[draw=drawColor,line width= 0.6pt,line join=round] ( 52.14, 34.30) --
	( 76.26, 43.88) --
	(100.38, 53.42) --
	(124.50, 62.22) --
	(148.61, 70.86) --
	(172.73, 80.18) --
	(196.85, 89.30);
\definecolor{drawColor}{RGB}{37,122,164}

\path[draw=drawColor,line width= 0.6pt,dash pattern=on 2pt off 2pt ,line join=round] ( 52.14, 44.24) --
	( 76.26, 58.44) --
	(100.38, 75.64) --
	(124.50, 91.43) --
	(148.61,106.52);
\definecolor{fillColor}{RGB}{247,192,26}

\path[fill=fillColor] (172.73, 80.18) circle (  1.96);

\path[fill=fillColor] ( 52.14, 34.30) circle (  1.96);

\path[fill=fillColor] (148.61, 70.86) circle (  1.96);

\path[fill=fillColor] (124.50, 62.22) circle (  1.96);

\path[fill=fillColor] ( 76.26, 43.88) circle (  1.96);

\path[fill=fillColor] (196.85, 89.30) circle (  1.96);

\path[fill=fillColor] (100.38, 53.42) circle (  1.96);
\definecolor{fillColor}{RGB}{37,122,164}

\path[fill=fillColor] ( 52.14, 47.29) --
	( 54.79, 42.71) --
	( 49.50, 42.71) --
	cycle;

\path[fill=fillColor] (148.61,109.57) --
	(151.26,105.00) --
	(145.97,105.00) --
	cycle;

\path[fill=fillColor] (124.50, 94.48) --
	(127.14, 89.90) --
	(121.85, 89.90) --
	cycle;

\path[fill=fillColor] ( 76.26, 61.49) --
	( 78.90, 56.91) --
	( 73.62, 56.91) --
	cycle;

\path[fill=fillColor] (100.38, 78.69) --
	(103.02, 74.11) --
	( 97.74, 74.11) --
	cycle;
\end{scope}
\begin{scope}
\path[clip] (  0.00,  0.00) rectangle (209.58,115.63);
\definecolor{drawColor}{RGB}{0,0,0}

\path[draw=drawColor,line width= 0.6pt,line join=round] ( 44.91, 30.69) --
	( 44.91,110.13);

\path[draw=drawColor,line width= 0.6pt,line join=round] ( 46.33,107.67) --
	( 44.91,110.13) --
	( 43.49,107.67);
\end{scope}
\begin{scope}
\path[clip] (  0.00,  0.00) rectangle (209.58,115.63);
\definecolor{drawColor}{gray}{0.30}

\node[text=drawColor,anchor=base east,inner sep=0pt, outer sep=0pt, scale=  0.88] at ( 39.96, 46.54) {1};

\node[text=drawColor,anchor=base east,inner sep=0pt, outer sep=0pt, scale=  0.88] at ( 39.96, 71.66) {100};

\node[text=drawColor,anchor=base east,inner sep=0pt, outer sep=0pt, scale=  0.88] at ( 39.96, 96.77) {10000};
\end{scope}
\begin{scope}
\path[clip] (  0.00,  0.00) rectangle (209.58,115.63);
\definecolor{drawColor}{gray}{0.20}

\path[draw=drawColor,line width= 0.6pt,line join=round] ( 42.16, 49.58) --
	( 44.91, 49.58);

\path[draw=drawColor,line width= 0.6pt,line join=round] ( 42.16, 74.69) --
	( 44.91, 74.69);

\path[draw=drawColor,line width= 0.6pt,line join=round] ( 42.16, 99.80) --
	( 44.91, 99.80);
\end{scope}
\begin{scope}
\path[clip] (  0.00,  0.00) rectangle (209.58,115.63);
\definecolor{drawColor}{RGB}{0,0,0}

\path[draw=drawColor,line width= 0.6pt,line join=round] ( 44.91, 30.69) --
	(204.08, 30.69);

\path[draw=drawColor,line width= 0.6pt,line join=round] (201.62, 29.26) --
	(204.08, 30.69) --
	(201.62, 32.11);
\end{scope}
\begin{scope}
\path[clip] (  0.00,  0.00) rectangle (209.58,115.63);
\definecolor{drawColor}{gray}{0.20}

\path[draw=drawColor,line width= 0.6pt,line join=round] ( 52.14, 27.94) --
	( 52.14, 30.69);

\path[draw=drawColor,line width= 0.6pt,line join=round] (100.38, 27.94) --
	(100.38, 30.69);

\path[draw=drawColor,line width= 0.6pt,line join=round] (148.61, 27.94) --
	(148.61, 30.69);

\path[draw=drawColor,line width= 0.6pt,line join=round] (196.85, 27.94) --
	(196.85, 30.69);
\end{scope}
\begin{scope}
\path[clip] (  0.00,  0.00) rectangle (209.58,115.63);
\definecolor{drawColor}{gray}{0.30}

\node[text=drawColor,anchor=base,inner sep=0pt, outer sep=0pt, scale=  0.88] at ( 52.14, 19.68) {4};

\node[text=drawColor,anchor=base,inner sep=0pt, outer sep=0pt, scale=  0.88] at (100.38, 19.68) {8};

\node[text=drawColor,anchor=base,inner sep=0pt, outer sep=0pt, scale=  0.88] at (148.61, 19.68) {12};

\node[text=drawColor,anchor=base,inner sep=0pt, outer sep=0pt, scale=  0.88] at (196.85, 19.68) {16};
\end{scope}
\begin{scope}
\path[clip] (  0.00,  0.00) rectangle (209.58,115.63);
\definecolor{drawColor}{RGB}{0,0,0}

\node[text=drawColor,anchor=base,inner sep=0pt, outer sep=0pt, scale=  1.10] at (124.50,  7.64) {$n$};
\end{scope}
\begin{scope}
\path[clip] (  0.00,  0.00) rectangle (209.58,115.63);
\definecolor{drawColor}{RGB}{0,0,0}

\node[text=drawColor,rotate= 90.00,anchor=base,inner sep=0pt, outer sep=0pt, scale=  1.10] at ( 13.08, 70.41) {time (ms)};
\end{scope}
\begin{scope}
\path[clip] (  0.00,  0.00) rectangle (209.58,115.63);

\path[] ( 41.77, 74.29) rectangle (102.16,114.20);
\end{scope}
\begin{scope}
\path[clip] (  0.00,  0.00) rectangle (209.58,115.63);
\definecolor{drawColor}{RGB}{247,192,26}

\path[draw=drawColor,line width= 0.6pt,line join=round] ( 48.72,101.47) -- ( 60.28,101.47);
\end{scope}
\begin{scope}
\path[clip] (  0.00,  0.00) rectangle (209.58,115.63);
\definecolor{fillColor}{RGB}{247,192,26}

\path[fill=fillColor] ( 54.50,101.47) circle (  1.96);
\end{scope}
\begin{scope}
\path[clip] (  0.00,  0.00) rectangle (209.58,115.63);
\definecolor{drawColor}{RGB}{37,122,164}

\path[draw=drawColor,line width= 0.6pt,dash pattern=on 2pt off 2pt ,line join=round] ( 48.72, 87.02) -- ( 60.28, 87.02);
\end{scope}
\begin{scope}
\path[clip] (  0.00,  0.00) rectangle (209.58,115.63);
\definecolor{fillColor}{RGB}{37,122,164}

\path[fill=fillColor] ( 54.50, 90.07) --
	( 57.14, 85.49) --
	( 51.86, 85.49) --
	cycle;
\end{scope}
\begin{scope}
\path[clip] (  0.00,  0.00) rectangle (209.58,115.63);
\definecolor{drawColor}{RGB}{0,0,0}

\node[text=drawColor,anchor=base west,inner sep=0pt, outer sep=0pt, scale=  0.80] at ( 67.23, 98.71) {DECOR};
\end{scope}
\begin{scope}
\path[clip] (  0.00,  0.00) rectangle (209.58,115.63);
\definecolor{drawColor}{RGB}{0,0,0}

\node[text=drawColor,anchor=base west,inner sep=0pt, outer sep=0pt, scale=  0.80] at ( 67.23, 84.26) {Naive};
\end{scope}
\end{tikzpicture}

%% file: files/plot-k=nsub1.tex
\begin{tikzpicture}[x=1pt,y=1pt]
\definecolor{fillColor}{RGB}{255,255,255}
\path[use as bounding box,fill=fillColor,fill opacity=0.00] (0,0) rectangle (209.58,115.63);
\begin{scope}
\path[clip] (  0.00,  0.00) rectangle (209.58,115.63);
\definecolor{drawColor}{RGB}{255,255,255}
\definecolor{fillColor}{RGB}{255,255,255}

\path[draw=drawColor,line width= 0.6pt,line join=round,line cap=round,fill=fillColor] (  0.00,  0.00) rectangle (209.58,115.63);
\end{scope}
\begin{scope}
\path[clip] ( 36.11, 30.69) rectangle (204.08,110.13);
\definecolor{fillColor}{RGB}{255,255,255}

\path[fill=fillColor] ( 36.11, 30.69) rectangle (204.08,110.13);
\definecolor{drawColor}{RGB}{247,192,26}

\path[draw=drawColor,line width= 0.6pt,line join=round] ( 43.75, 34.30) --
	( 69.20, 45.49) --
	( 94.65, 56.76) --
	(120.10, 66.98) --
	(145.55, 77.20) --
	(171.00, 88.03) --
	(196.45, 98.38);
\definecolor{drawColor}{RGB}{37,122,164}

\path[draw=drawColor,line width= 0.6pt,dash pattern=on 2pt off 2pt ,line join=round] ( 43.75, 39.50) --
	( 69.20, 46.63) --
	( 94.65, 62.13) --
	(120.10, 66.49) --
	(145.55, 84.49) --
	(171.00, 95.45) --
	(196.45,106.52);
\definecolor{fillColor}{RGB}{247,192,26}

\path[fill=fillColor] (171.00, 88.03) circle (  1.96);

\path[fill=fillColor] ( 43.75, 34.30) circle (  1.96);

\path[fill=fillColor] (145.55, 77.20) circle (  1.96);

\path[fill=fillColor] (120.10, 66.98) circle (  1.96);

\path[fill=fillColor] ( 69.20, 45.49) circle (  1.96);

\path[fill=fillColor] (196.45, 98.38) circle (  1.96);

\path[fill=fillColor] ( 94.65, 56.76) circle (  1.96);
\definecolor{fillColor}{RGB}{37,122,164}

\path[fill=fillColor] (171.00, 98.50) --
	(173.64, 93.92) --
	(168.36, 93.92) --
	cycle;

\path[fill=fillColor] ( 43.75, 42.55) --
	( 46.39, 37.97) --
	( 41.10, 37.97) --
	cycle;

\path[fill=fillColor] (145.55, 87.54) --
	(148.19, 82.96) --
	(142.90, 82.96) --
	cycle;

\path[fill=fillColor] (120.10, 69.55) --
	(122.74, 64.97) --
	(117.45, 64.97) --
	cycle;

\path[fill=fillColor] ( 69.20, 49.68) --
	( 71.84, 45.10) --
	( 66.55, 45.10) --
	cycle;

\path[fill=fillColor] (196.45,109.57) --
	(199.09,105.00) --
	(193.81,105.00) --
	cycle;

\path[fill=fillColor] ( 94.65, 65.18) --
	( 97.29, 60.60) --
	( 92.00, 60.60) --
	cycle;
\end{scope}
\begin{scope}
\path[clip] (  0.00,  0.00) rectangle (209.58,115.63);
\definecolor{drawColor}{RGB}{0,0,0}

\path[draw=drawColor,line width= 0.6pt,line join=round] ( 36.11, 30.69) --
	( 36.11,110.13);

\path[draw=drawColor,line width= 0.6pt,line join=round] ( 37.53,107.67) --
	( 36.11,110.13) --
	( 34.69,107.67);
\end{scope}
\begin{scope}
\path[clip] (  0.00,  0.00) rectangle (209.58,115.63);
\definecolor{drawColor}{gray}{0.30}

\node[text=drawColor,anchor=base east,inner sep=0pt, outer sep=0pt, scale=  0.88] at ( 31.16, 48.67) {1};

\node[text=drawColor,anchor=base east,inner sep=0pt, outer sep=0pt, scale=  0.88] at ( 31.16, 79.05) {100};
\end{scope}
\begin{scope}
\path[clip] (  0.00,  0.00) rectangle (209.58,115.63);
\definecolor{drawColor}{gray}{0.20}

\path[draw=drawColor,line width= 0.6pt,line join=round] ( 33.36, 51.70) --
	( 36.11, 51.70);

\path[draw=drawColor,line width= 0.6pt,line join=round] ( 33.36, 82.08) --
	( 36.11, 82.08);
\end{scope}
\begin{scope}
\path[clip] (  0.00,  0.00) rectangle (209.58,115.63);
\definecolor{drawColor}{RGB}{0,0,0}

\path[draw=drawColor,line width= 0.6pt,line join=round] ( 36.11, 30.69) --
	(204.08, 30.69);

\path[draw=drawColor,line width= 0.6pt,line join=round] (201.62, 29.26) --
	(204.08, 30.69) --
	(201.62, 32.11);
\end{scope}
\begin{scope}
\path[clip] (  0.00,  0.00) rectangle (209.58,115.63);
\definecolor{drawColor}{gray}{0.20}

\path[draw=drawColor,line width= 0.6pt,line join=round] ( 43.75, 27.94) --
	( 43.75, 30.69);

\path[draw=drawColor,line width= 0.6pt,line join=round] ( 94.65, 27.94) --
	( 94.65, 30.69);

\path[draw=drawColor,line width= 0.6pt,line join=round] (145.55, 27.94) --
	(145.55, 30.69);

\path[draw=drawColor,line width= 0.6pt,line join=round] (196.45, 27.94) --
	(196.45, 30.69);
\end{scope}
\begin{scope}
\path[clip] (  0.00,  0.00) rectangle (209.58,115.63);
\definecolor{drawColor}{gray}{0.30}

\node[text=drawColor,anchor=base,inner sep=0pt, outer sep=0pt, scale=  0.88] at ( 43.75, 19.68) {4};

\node[text=drawColor,anchor=base,inner sep=0pt, outer sep=0pt, scale=  0.88] at ( 94.65, 19.68) {8};

\node[text=drawColor,anchor=base,inner sep=0pt, outer sep=0pt, scale=  0.88] at (145.55, 19.68) {12};

\node[text=drawColor,anchor=base,inner sep=0pt, outer sep=0pt, scale=  0.88] at (196.45, 19.68) {16};
\end{scope}
\begin{scope}
\path[clip] (  0.00,  0.00) rectangle (209.58,115.63);
\definecolor{drawColor}{RGB}{0,0,0}

\node[text=drawColor,anchor=base,inner sep=0pt, outer sep=0pt, scale=  1.10] at (120.10,  7.64) {$n$};
\end{scope}
\begin{scope}
\path[clip] (  0.00,  0.00) rectangle (209.58,115.63);
\definecolor{drawColor}{RGB}{0,0,0}

\node[text=drawColor,rotate= 90.00,anchor=base,inner sep=0pt, outer sep=0pt, scale=  1.10] at ( 13.08, 70.41) {time (ms)};
\end{scope}
\begin{scope}
\path[clip] (  0.00,  0.00) rectangle (209.58,115.63);

\path[] ( 34.47, 74.29) rectangle ( 94.86,114.20);
\end{scope}
\begin{scope}
\path[clip] (  0.00,  0.00) rectangle (209.58,115.63);
\definecolor{drawColor}{RGB}{247,192,26}

\path[draw=drawColor,line width= 0.6pt,line join=round] ( 41.42,101.47) -- ( 52.98,101.47);
\end{scope}
\begin{scope}
\path[clip] (  0.00,  0.00) rectangle (209.58,115.63);
\definecolor{fillColor}{RGB}{247,192,26}

\path[fill=fillColor] ( 47.20,101.47) circle (  1.96);
\end{scope}
\begin{scope}
\path[clip] (  0.00,  0.00) rectangle (209.58,115.63);
\definecolor{drawColor}{RGB}{37,122,164}

\path[draw=drawColor,line width= 0.6pt,dash pattern=on 2pt off 2pt ,line join=round] ( 41.42, 87.02) -- ( 52.98, 87.02);
\end{scope}
\begin{scope}
\path[clip] (  0.00,  0.00) rectangle (209.58,115.63);
\definecolor{fillColor}{RGB}{37,122,164}

\path[fill=fillColor] ( 47.20, 90.07) --
	( 49.84, 85.49) --
	( 44.56, 85.49) --
	cycle;
\end{scope}
\begin{scope}
\path[clip] (  0.00,  0.00) rectangle (209.58,115.63);
\definecolor{drawColor}{RGB}{0,0,0}

\node[text=drawColor,anchor=base west,inner sep=0pt, outer sep=0pt, scale=  0.80] at ( 59.92, 98.71) {DECOR};
\end{scope}
\begin{scope}
\path[clip] (  0.00,  0.00) rectangle (209.58,115.63);
\definecolor{drawColor}{RGB}{0,0,0}

\node[text=drawColor,anchor=base west,inner sep=0pt, outer sep=0pt, scale=  0.80] at ( 59.92, 84.26) {Naive};
\end{scope}
\end{tikzpicture}

%% file: files/plot-k=n.tex
\begin{tikzpicture}[x=1pt,y=1pt]
\definecolor{fillColor}{RGB}{255,255,255}
\path[use as bounding box,fill=fillColor,fill opacity=0.00] (0,0) rectangle (209.58,115.63);
\begin{scope}
\path[clip] (  0.00,  0.00) rectangle (209.58,115.63);
\definecolor{drawColor}{RGB}{255,255,255}
\definecolor{fillColor}{RGB}{255,255,255}

\path[draw=drawColor,line width= 0.6pt,line join=round,line cap=round,fill=fillColor] (  0.00,  0.00) rectangle (209.58,115.63);
\end{scope}
\begin{scope}
\path[clip] ( 36.11, 30.69) rectangle (204.08,110.13);
\definecolor{fillColor}{RGB}{255,255,255}

\path[fill=fillColor] ( 36.11, 30.69) rectangle (204.08,110.13);
\definecolor{drawColor}{RGB}{247,192,26}

\path[draw=drawColor,line width= 0.6pt,line join=round] ( 43.75, 35.63) --
	( 65.56, 45.60) --
	( 87.38, 56.06) --
	(109.19, 66.70) --
	(131.00, 76.58) --
	(152.82, 86.28) --
	(174.63, 96.61) --
	(196.45,106.52);
\definecolor{drawColor}{RGB}{37,122,164}

\path[draw=drawColor,line width= 0.6pt,dash pattern=on 2pt off 2pt ,line join=round] ( 43.75, 34.30) --
	( 65.56, 41.33) --
	( 87.38, 49.88) --
	(109.19, 59.44) --
	(131.00, 68.97) --
	(152.82, 78.31) --
	(174.63, 87.86) --
	(196.45, 97.27);
\definecolor{fillColor}{RGB}{247,192,26}

\path[fill=fillColor] (174.63, 96.61) circle (  1.96);

\path[fill=fillColor] ( 65.56, 45.60) circle (  1.96);

\path[fill=fillColor] (152.82, 86.28) circle (  1.96);

\path[fill=fillColor] (131.00, 76.58) circle (  1.96);

\path[fill=fillColor] ( 43.75, 35.63) circle (  1.96);

\path[fill=fillColor] ( 87.38, 56.06) circle (  1.96);

\path[fill=fillColor] (196.45,106.52) circle (  1.96);

\path[fill=fillColor] (109.19, 66.70) circle (  1.96);
\definecolor{fillColor}{RGB}{37,122,164}

\path[fill=fillColor] (174.63, 90.91) --
	(177.28, 86.33) --
	(171.99, 86.33) --
	cycle;

\path[fill=fillColor] ( 65.56, 44.38) --
	( 68.20, 39.80) --
	( 62.92, 39.80) --
	cycle;

\path[fill=fillColor] (152.82, 81.36) --
	(155.46, 76.78) --
	(150.18, 76.78) --
	cycle;

\path[fill=fillColor] (131.00, 72.03) --
	(133.65, 67.45) --
	(128.36, 67.45) --
	cycle;

\path[fill=fillColor] ( 43.75, 37.35) --
	( 46.39, 32.77) --
	( 41.10, 32.77) --
	cycle;

\path[fill=fillColor] ( 87.38, 52.93) --
	( 90.02, 48.35) --
	( 84.73, 48.35) --
	cycle;

\path[fill=fillColor] (196.45,100.32) --
	(199.09, 95.75) --
	(193.81, 95.75) --
	cycle;

\path[fill=fillColor] (109.19, 62.49) --
	(111.83, 57.91) --
	(106.55, 57.91) --
	cycle;
\end{scope}
\begin{scope}
\path[clip] (  0.00,  0.00) rectangle (209.58,115.63);
\definecolor{drawColor}{RGB}{0,0,0}

\path[draw=drawColor,line width= 0.6pt,line join=round] ( 36.11, 30.69) --
	( 36.11,110.13);

\path[draw=drawColor,line width= 0.6pt,line join=round] ( 37.53,107.67) --
	( 36.11,110.13) --
	( 34.69,107.67);
\end{scope}
\begin{scope}
\path[clip] (  0.00,  0.00) rectangle (209.58,115.63);
\definecolor{drawColor}{gray}{0.30}

\node[text=drawColor,anchor=base east,inner sep=0pt, outer sep=0pt, scale=  0.88] at ( 31.16, 58.89) {1};

\node[text=drawColor,anchor=base east,inner sep=0pt, outer sep=0pt, scale=  0.88] at ( 31.16, 87.93) {100};
\end{scope}
\begin{scope}
\path[clip] (  0.00,  0.00) rectangle (209.58,115.63);
\definecolor{drawColor}{gray}{0.20}

\path[draw=drawColor,line width= 0.6pt,line join=round] ( 33.36, 61.92) --
	( 36.11, 61.92);

\path[draw=drawColor,line width= 0.6pt,line join=round] ( 33.36, 90.96) --
	( 36.11, 90.96);
\end{scope}
\begin{scope}
\path[clip] (  0.00,  0.00) rectangle (209.58,115.63);
\definecolor{drawColor}{RGB}{0,0,0}

\path[draw=drawColor,line width= 0.6pt,line join=round] ( 36.11, 30.69) --
	(204.08, 30.69);

\path[draw=drawColor,line width= 0.6pt,line join=round] (201.62, 29.26) --
	(204.08, 30.69) --
	(201.62, 32.11);
\end{scope}
\begin{scope}
\path[clip] (  0.00,  0.00) rectangle (209.58,115.63);
\definecolor{drawColor}{gray}{0.20}

\path[draw=drawColor,line width= 0.6pt,line join=round] ( 65.56, 27.94) --
	( 65.56, 30.69);

\path[draw=drawColor,line width= 0.6pt,line join=round] (109.19, 27.94) --
	(109.19, 30.69);

\path[draw=drawColor,line width= 0.6pt,line join=round] (152.82, 27.94) --
	(152.82, 30.69);

\path[draw=drawColor,line width= 0.6pt,line join=round] (196.45, 27.94) --
	(196.45, 30.69);
\end{scope}
\begin{scope}
\path[clip] (  0.00,  0.00) rectangle (209.58,115.63);
\definecolor{drawColor}{gray}{0.30}

\node[text=drawColor,anchor=base,inner sep=0pt, outer sep=0pt, scale=  0.88] at ( 65.56, 19.68) {4};

\node[text=drawColor,anchor=base,inner sep=0pt, outer sep=0pt, scale=  0.88] at (109.19, 19.68) {8};

\node[text=drawColor,anchor=base,inner sep=0pt, outer sep=0pt, scale=  0.88] at (152.82, 19.68) {12};

\node[text=drawColor,anchor=base,inner sep=0pt, outer sep=0pt, scale=  0.88] at (196.45, 19.68) {16};
\end{scope}
\begin{scope}
\path[clip] (  0.00,  0.00) rectangle (209.58,115.63);
\definecolor{drawColor}{RGB}{0,0,0}

\node[text=drawColor,anchor=base,inner sep=0pt, outer sep=0pt, scale=  1.10] at (120.10,  7.64) {$n$};
\end{scope}
\begin{scope}
\path[clip] (  0.00,  0.00) rectangle (209.58,115.63);
\definecolor{drawColor}{RGB}{0,0,0}

\node[text=drawColor,rotate= 90.00,anchor=base,inner sep=0pt, outer sep=0pt, scale=  1.10] at ( 13.08, 70.41) {time (ms)};
\end{scope}
\begin{scope}
\path[clip] (  0.00,  0.00) rectangle (209.58,115.63);

\path[] ( 34.47, 74.29) rectangle ( 94.86,114.20);
\end{scope}
\begin{scope}
\path[clip] (  0.00,  0.00) rectangle (209.58,115.63);
\definecolor{drawColor}{RGB}{247,192,26}

\path[draw=drawColor,line width= 0.6pt,line join=round] ( 41.42,101.47) -- ( 52.98,101.47);
\end{scope}
\begin{scope}
\path[clip] (  0.00,  0.00) rectangle (209.58,115.63);
\definecolor{fillColor}{RGB}{247,192,26}

\path[fill=fillColor] ( 47.20,101.47) circle (  1.96);
\end{scope}
\begin{scope}
\path[clip] (  0.00,  0.00) rectangle (209.58,115.63);
\definecolor{drawColor}{RGB}{37,122,164}

\path[draw=drawColor,line width= 0.6pt,dash pattern=on 2pt off 2pt ,line join=round] ( 41.42, 87.02) -- ( 52.98, 87.02);
\end{scope}
\begin{scope}
\path[clip] (  0.00,  0.00) rectangle (209.58,115.63);
\definecolor{fillColor}{RGB}{37,122,164}

\path[fill=fillColor] ( 47.20, 90.07) --
	( 49.84, 85.49) --
	( 44.56, 85.49) --
	cycle;
\end{scope}
\begin{scope}
\path[clip] (  0.00,  0.00) rectangle (209.58,115.63);
\definecolor{drawColor}{RGB}{0,0,0}

\node[text=drawColor,anchor=base west,inner sep=0pt, outer sep=0pt, scale=  0.80] at ( 59.92, 98.71) {DECOR};
\end{scope}
\begin{scope}
\path[clip] (  0.00,  0.00) rectangle (209.58,115.63);
\definecolor{drawColor}{RGB}{0,0,0}

\node[text=drawColor,anchor=base west,inner sep=0pt, outer sep=0pt, scale=  0.80] at ( 59.92, 84.26) {Naive};
\end{scope}
\end{tikzpicture}